%% file: S3Attention_arxiv.tex
\documentclass[lettersize,journal]{IEEEtran}
\usepackage{amsmath,amsfonts}
\usepackage{algorithmic}
\usepackage{algorithm}
\usepackage{array}
\usepackage{textcomp}
\usepackage{stfloats}
\usepackage{url}
\usepackage{verbatim}
\usepackage{graphicx}
\usepackage[numbers,sort&compress]{natbib}
\usepackage{subcaption}
\usepackage{wrapfig}

\usepackage[utf8]{inputenc} 
\usepackage[T1]{fontenc}    
\usepackage{booktabs}       
\usepackage{nicefrac}       
\usepackage{microtype}      
\usepackage{xcolor}         
\usepackage{bm,bbm}
\usepackage{amsthm}
\usepackage{graphicx}
\usepackage{multirow, paralist}
\usepackage[british]{babel}
\usepackage{enumitem}
\usepackage{xcolor}

\renewcommand{\thetable}{\arabic{table}}

\usepackage{algorithm}
\usepackage{listings}
\newtheorem{definition}{Definition}

\newtheorem{prop}{Proposition}
\usepackage[hidelinks]{hyperref}       
\usepackage{url}            

\begin{document}

\title{S$^3$Attention: Improving Long Sequence Attention with Smoothed Skeleton Sketching}

\author{~\\Xue Wang, Tian Zhou, Jianqing Zhu, Jialin Liu, Kun Yuan, Tao Yao, Wotao Yin, Rong Jin, HanQin~Cai\\~\\
\thanks{
This work was partially supported by NSF DMS 2304489. 
}
\thanks{X.~Wang, T.~Zhou, and W.~Yin are with Alibaba Group, Bellevue, WA 98004, USA (e-mail: \{xue.w, tian.zt, wotao.yin\}@alibaba-inc.com).}%
\thanks{J.~Zhu is with Computer, Electrical and Mathematical Science and Engineering Division, King Abdullah University of Science and Technology, Thuwal 23955, Saudi Arabia (e-mail: zjq941026@gmail.com).}%
\thanks{J.~Liu is with the Department of Statistics and Data Science, University of Central Florida, Orlando, FL 32816, USA (e-mail: jialin.liu@ucf.edu).}%
\thanks{K.~Yuan is with Center for Machine Learning Research, Peking University, Beijing 100871, China (e-mail: kunyuan@pku.edu.cn).}%
\thanks{T.~Yao is with Antai College of Economics and Management, Shanghai Jiao Tong University, Shanghai 200030, China (e-mail: taoyao@sjtu.edu.cn).}%
\thanks{R.~Jin is with Meta, Menlo Park, CA 94025, USA (e-mail: rongjinemail@gmail.com).}%
\thanks{H.Q.~Cai is with the Department of Statistics and Data Science and the Department of Computer Science, University of Central Florida, Orlando, FL 32816, USA (Corresponding author, e-mail: hqcai@ucf.edu).}%
}

\markboth{}%
{Wang \MakeLowercase{\textit{et al.}}: S$^3$Attention}


\maketitle

\begin{abstract}
Attention based models have achieved many remarkable breakthroughs in numerous applications. However, the quadratic complexity of Attention makes the vanilla Attention based models hard to apply to long sequence tasks. Various improved Attention structures are proposed to reduce the computation cost by inducing low rankness and approximating the whole sequence by sub-sequences. The most challenging part of those approaches is maintaining the proper balance between information preservation and computation reduction: the longer sub-sequences used, the better information is preserved, but at the price of introducing more noise and computational costs. In this paper, we propose a smoothed skeleton sketching based Attention structure, coined S$^3$Attention, which significantly improves upon the previous attempts to negotiate this trade-off. S$^3$Attention has two mechanisms to effectively minimize the impact of noise while keeping the linear complexity to the sequence length: a smoothing block to mix information over long sequences and a matrix sketching method that simultaneously selects columns and rows from the input matrix. We verify the effectiveness of S$^3$Attention both theoretically and empirically. Extensive studies over Long Range Arena (LRA) datasets and six time-series forecasting show that S$^3$Attention significantly outperforms both vanilla Attention and other state-of-the-art variants of Attention structures. 
\vspace{0.15in}
\end{abstract}

\begin{IEEEkeywords}
Deep Learning, Attention, Skeleton Approximation, Dimensionality Reduction.
\vspace{0.15in}
\end{IEEEkeywords}

\section{Introduction}
\label{Introduction}

\IEEEPARstart{M}{odern} Attention based models, first introduced in \citep{vaswani2017attention}, have made significant contributions in several areas of artificial intelligence, including natural language processing (NLP) \citep{devlin2018bert,brown2020language,liu2019roberta,clark2020electra}, computer vision (CV) \citep{dosovitskiy2020image,liu2021swin,touvron2021training,yuan2021volo,zhou2021elsa}, and time series forecasting \citep{xu2021autoformer,FedFormer}. These models have also been applied in a variety of long sequences mining tasks, such as traffic \citep{deshpande2021long,elmi2021deepfec}, health care \citep{lim2018forecasting,zhang2018multi}, retail \citep{bose2017probabilistic,chauhan2020time}, security industry \citep{surveillance}, and web applications \citep{Faloutsos2020wwwtut,hou2022multi,zhou2020domain}. The Attention scheme efficiently captures long-term global and short-term local correlations when the length of the token sequences is relatively small. However, due to the quadratic complexity of standard Attention, many approaches have been developed to reduce the computational complexity for long sequences (e.g., \citep{zhu2021long}). Most of them try to exploit the special patterns of the Attention matrix, such as low rankness, locality, sparsity, or graph structures. One group of approaches is to build a linear approximation for the $\mathrm{softmax}$ operator (e.g., \citep{choromanski2020rethinking,chen2021scatterbrain,chowdhury2021learning,qin2021cosformer}). Despite the efficiency of the linear approximation, these approximation methods often perform worse than the original $\mathrm{softmax}$ based Attention. More discussion of efficient Attention-based models for long sequences can be found in the section of related work.

In this work, we focus on approaches that assume a low-rank structure of the input matrix. They approximate the global information in long sequences by sub-sequences (i.e., short sequences) of landmarks, and only compute Attention between queries and selected landmarks (e.g., \citep{zhu2021long,zhu2021h,nguyen2021fmmformer,ma2021luna}). Although those models enjoy linear computational cost and often better performance than vanilla Attention, they face one major challenge, i.e., how to balance between information preservation and noise reduction. By choosing a larger number of landmarks, we are able to preserve more global information but at the price of introducing more noise into the sequential model and more computational cost. 

In this work, we propose an efficient Attention architecture, termed S$^3$Attention, that introduces two mechanisms to address the balance explicitly as illustrated in Figure 1. First, we introduce a smoothing block into the Attention architecture. It effectively mixes global information over the long sequences by Fourier Transformation and local information over the sequences by a convolution kernel. Through the information mixture, we can reduce the noise for individual tokens over the sequences, and at the same time, improve their representativeness for the entire sequences. Second, we introduce a matrix sketching technique to approximate the input matrix by a smaller number of rows and columns. Standard Attention can be seen as reweighing the columns of the value matrix. Important columns are assigned high attention weights and remain in the output matrix, while small attention weights eliminate insignificant columns. The Attention mechanism is equivalent to column selection if we replace the $\mathrm{softmax}$ operator with the corresponding $\mathrm{argmax}$ operator. However, sampling only columns may not generate a good summary of the matrix, and could be subjected to the noises in individual columns. We address this problem by exploiting Skeleton Sketching technique \citep{drineas2008relative,chiu2013sublinear,cai2021ircur,cai2021rcur,cai2021modewisedecomp,hamm2022RieCUR,cai2023ccs,cai2024robust,cai2024rccs} in the matrix approximation community, which is also known as CUR approximation. Theoretically, for a rank-$r$ matrix $\bm{X}\in\mathbbm{R}^{n\times d}$, we can take $\mathcal{O}(r\log d)$ column samples and $\mathcal{O}(r\log n)$ row samples to construct a so-called Skeleton approximation $\bm{X} \approx \bm{C}\bm{U}\bm{R}$, where $\bm{C}$ and $\bm{R}$ are matrices consisting of sampled columns and rows of $\bm{X}$ respectively, and $\bm{U}$ is the pseudo-inverse of $\bm{C}$ and $\bm{R}$'s intersection. By combining these mechanisms, we found, both theoretically and empirically, that S$^3$Attention is able to preserve global information over long sequences and reduce the impact of noise simultaneously, thus leading to better performance than state-of-the-art variants of Attention based models for long sequences, without having to sacrifice the linear complexity w.r.t.~sequence length.

In short, we summarize our main contributions as follows: 
\begin{enumerate}[leftmargin=0.3in, label=\arabic*.]
\item We propose the S$^3$Attention, an efficient model that integrates smoother column Attention and row Attention components to unfold a randomized linear matrix sketching algorithm.

\item By randomly selecting a fixed number of rows and columns, the proposed model achieves near-linear computational complexity and memory cost. The effectiveness of this selection method is verified both theoretically and empirically.

\item We conduct extensive experiments over Long-term sequences, long-term time series forecasting and GLUE tasks. In particular, the Long Range Arena benchmark \citep{tay2021long}, achieves an average accuracy of 64\% and 66\% with fixed parameters (suggested setting in \citep{tay2021long,mathieu2014fast}) and fine-tuned parameters respectively. It improves from 62\% of the best Attention-type model. Moreover, it also has a comparable performance with the recent state-of-the-art long-term time series forecasting models for long-term time series forecasting and GLUE tasks. 
\end{enumerate}

\vspace{0.05in}\noindent\textbf{Organization.}~~ We structure the rest of this paper as follows: In Section~\ref{sec:related work}, we briefly review the relevant literature related to efficient Attention based models. Section~\ref{sec:models} introduces the model structure and performs a theoretical analysis to justify the proposed model. We empirically verify the efficiency and accuracy of S$^3$Attention in Section~\ref{sec:experiment} we discuss limitations and future directions in Section~\ref{sec:conclusion}. The experimental details are provided in the supplementary material.

\section{Related Work} \label{sec:related work}

This section provides an overview of the literature mainly focusing on efficient Attention based models. The techniques include sparse or local Attention, low rankness, and kernel approximation. We refer the reader interested in their details to the survey \citep{tay2020efficient}.

\vspace{0.05in}\noindent\textbf{Sparse Attention.}~~ The general idea of these methods is restricting the query token to perform Attention only within a specific small region, such as its local region or some global tokens. In this setting, the Attention matrix becomes sparse compared to the original one.   \citep{qiu2019blockwise} proposes BlockBert, which introduces sparse block structures into the Attention matrix by multiplying a masking matrix. \citep{parmar2018image} applies Attention within blocks for the image generation task. \citep{liu2018generating} divides the whole sequences into blocks and uses a stride convolution to reduce the model complexity. However, these block-type Transformers ignore the connections among blocks. To address this issue, Transformer-XL \citep{dai2019transformer} and Compressive Transformer \citep{rae2019compressive} propose a recurrence mechanism to connect multiple blocks.  Transformer-LS \citep{zhu2021long} combines local Attention with a dynamic projection to capture long-term dependence. 
\citep{tay2020sparse} uses a meta-sorting network to permute over sequences and quasi-global Attention with local windows to improve memory efficiency.

Another approach in this category is based on stride Attention. Longformer \citep{beltagy2020longformer} uses dilated sliding windows to obtain a sparse Attention matrix. Sparse Transformers \citep{child2019generating} consider approximating a dense Attention matrix by several sparse factorization methods. In addition, some methods reduce the complexity by clustering tokens. For example, Reformer \citep{kitaev2020reformer} uses a hash similarity measure to cluster tokens, and Routing Transformer \citep{roy2021efficient} uses k-means to cluster tokens. BigBird \citep{zaheer2020big} proposes a generalized Attention mechanism described by a directed graph to reduce Attention complexity.
\citep{lee2021fnet} considers using 2D Fourier Transformation to mix the token matrix directly. \citep{tan2021ponet} uses max pooling scheme to reduce the computation costs.

\vspace{0.05in}\noindent\textbf{Low-rank and Kernel Methods.}~~ Inducing low rankness into the Attention matrix can quickly reduce the complexity and the kernel approximation is widely applied in efficient low-rank approximation. Linformer \citep{wang2020linformer} and Luna \citep{ma2021luna} approximate $\mathrm{softmax}$  with linear functions, which yield a linear time and space complexity. \citep{choromanski2020rethinking,peng2021random} use random features tricks and reach promising numerical performance.  \citep{winata2020lightweight} proposes Low-Rank Transformer based on matrix factorization. FMMformer \citep{nguyen2021fmmformer} combines the fast multipole method with the kernel method. Synthesizer \citep{tay2005synthesizer} uses a random low-rank matrix to replace the Attention matrix. Nystr{\"o}mformer \citep{xiong2021nystromformer} adopts the Nystr{\"o}m method to approximate standard Attention. Linear Transformer \citep{katharopoulos2020transformers} expresses Attention as a linear dot-product of kernel feature maps. \citep{zhu2021h} applies the Multigrid method to efficiently compute the Attention matrix recursively. Cosformer \citep{qin2021cosformer} develops a cosine-based re-weighting mechanism to linearize the $\mathrm{softmax}$ function. \citep{chen2021scatterbrain} proposes the Scatterbrain, which unifies locality-sensitive hashing and the kernel method into Attention for accurate and efficient approximation.

\vspace{0.05in}\noindent\textbf{Sequence Length Reduction.}~~  Reducing sequence length is an efficient means to reduce computational costs. Perceiver IO  \citep{jaegle2021perceiver} encodes inputs to a latent space whose size is typically smaller than the inputs and outputs, making the process computationally scalable to even very large inputs and outputs. Funnel-Transformer \citep{dai2020funnel} uses pooling tricks to reduce sequence length, significantly saving both FLOPs and memory. Swin Transformer \citep{liu2021swin} proposes a shifted windows method in vision Transformer, and the resulting complexity is linear in the image size. XCiT \citep{ali2021xcit} considers an Attention over the hidden dimension for vision tasks. Charformer \citep{tay2021charformer} downsamples the sequences of words to construct latent subwords.

\vspace{0.05in}\noindent\textbf{Improved Recurrent Neural Network.}~~ Another research track to solve the long sequence tasks is to improve the Recurrent Neural Network (RNN). \citep{gu2020hippo} propose to view RNN as the state space model and use optimal polynomial projections to improve its memorization ability. The follow-up works \citep{gu2021efficiently,gupta2022diagonal,gu2021combining,smith2023simplified,gupta2022diagonal,hasani2023liquid,gu2023how} consider more sophisticated designs and obtain better performance and higher efficiency. The concurrent work \citep{ma2023mega} proposes coupling the moving average type of RNN structure with single-headed Attention and reaches quite promising results. Different from those works, in this paper, we consider a different research angle and focus on making improvements inside the Attention structure. 

\begin{figure}[h]
\centering

\includegraphics[width=\linewidth]{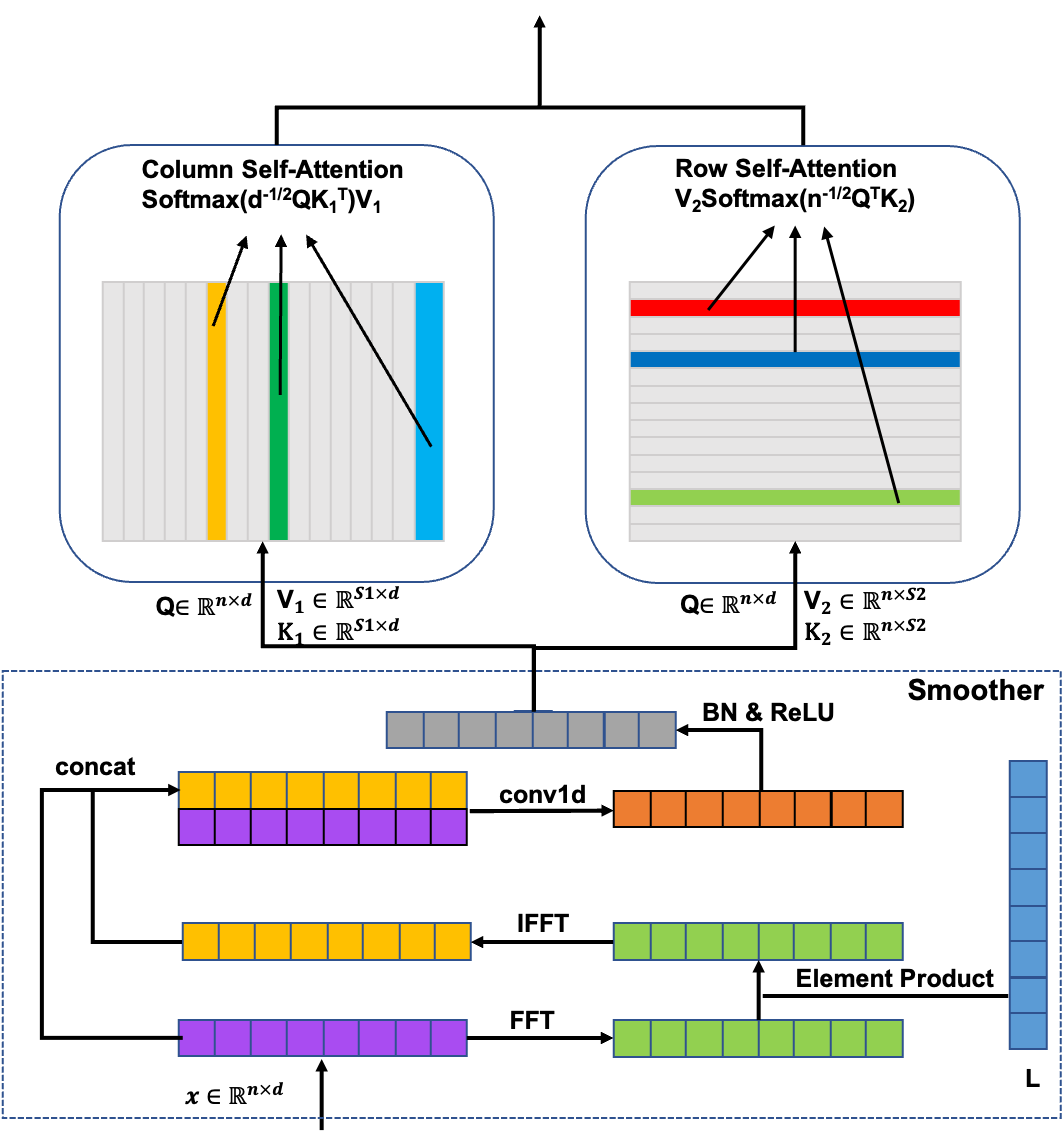}
\caption{Illustration of the architecture of S$^3$Attention. }
\label{fig:boost}
\end{figure}

\section{S$^3$Attention} \label{sec:models}
We start by going over the vanilla Attention. For a sequence of length $n$, the vanilla Attention is dot-product type \citep{vaswani2017attention}. Following standard notation, the Attention matrix $\bm{A}\in\mathbb{R}^{n\times n}$ is defined as:
\begin{align*}
    \bm{A} = \mathrm{softmax}\left(\frac{1}{\sqrt{d}}\bm{Q}\bm{K}^\top \right),
\end{align*}
where $\bm{Q}\in\mathbb{R}^{n\times d}$ denotes the queries while $\bm{K}\in\mathbbm{R}^{n\times d}$ denotes the keys, and $d$ represents the hidden dimension. By multiplying the attention weights $\bm{A}$ with the values $\bm{V}\in\mathbb{R}^{n\times d}$, we can calculate the new values as $\hat{\bm{V}} = \bm{A}\bm{V}$.

Intuitively, the Attention is the weighted average over the old ones, where the weights are defined by the Attention matrix $\bm{A}$. In this paper, we consider generating $\bm{Q}$, $\bm{K}$ and $\bm{V}$ via the linear projection of the input
token matrix $\bm{X}$:
\begin{align}
    \bm{Q} = \bm{X}\bm{W}_Q,\ \bm{K} = \bm{X}\bm{W}_K,\ \bm{V}=\bm{X}\bm{W}_V,\notag
\end{align}
where $\bm{X}\in\mathbb{R}^{n\times d}$ and $\bm{W}_Q,\bm{W}_K,\bm{W}_V\in\mathbb{R}^{d\times d}$.

The vanilla procedure has two drawbacks in concentrating the information from $\bm{V}$. First, when computing the $\bm{Q}\bm{K}^{\top}$ part, full dense matrix multiplication is involved at a cost of $\mathcal{O}(n^2)$ vector multiplications. 
It can be prohibitive for long sequence problems. On the other hand, if we view the $\mathrm{softmax}$ operator as an approximation of the argmax counterpart, $\hat{\bm{V}}$ becomes a row selection from $\bm{V}$. From the view of matrix sketching, solely using either row or column selection could reduce the cost-effectiveness, compared to using both column and row selections.

\subsection{Skeleton Sketching Based Attention}

We propose a Skeleton Sketching induced Attention structure to address those issues. First, we modify the original Attention to build the column Attention as follows:
\begin{align}
     \hat{\bm{V}}_1 = \mathrm{softmax}\left(\frac{1}{\sqrt{d}}\bm{Q}\bm{K}^\top \bm{P}_1^\top \right)\bm{P}_1\bm{V},\notag
\end{align}
where $\bm{P}_1\in\mathbbm{R}^{s_1\times n}$ denotes the sampling matrix and $s_1$ is the number of columns sampled. Let $i_1<i_2<...<i_{s_1}$ be the indices of the randomly sampled columns. Let $\bm{P}_{1,ab}$ denote the element located at the $a$-th column and $b$-th row and we have $\bm{P}_{1,ab} = 1$ if $i_a = b$ and $0$ otherwise. By these constructions, we can reduce the computational cost to $\mathcal{O}(ns_1 d^2 + ns_1^2d)$.

Similarly, we build the row sampling matrix $\bm{P}_2\in\mathbbm{R}^{d\times s_2}$ indicating the locations of the $s_2$ sample rows. Compute the row Attention as:
\begin{align}
     \hat{\bm{V}}_2 =\bm{V}\bm{P}_2~\mathrm{softmax}\left(\frac{1}{\sqrt{n}}\bm{P}_2^\top \bm{K}^\top \bm{Q}\right).\notag
\end{align}
Finally, we apply the layer-norm on $\hat{\bm{V}}_1$ and $\hat{\bm{V}}_2$ and then add them together to generate the final output:
\begin{align}
    \hat{\bm{V}} =  \mathrm{layernorm}_1(\hat{\bm{V}}_1)+ \mathrm{layernorm}_2(\hat{\bm{V}}_2).\label{eq:cur_output}
\end{align}
The usage of layer norm is to balance the output scales of column and row Attentions. A similar trick has been used in \citep{zhu2021long}, where the layer norm is applied to resolve scale mismatches between the different Attention mechanisms.

Before going into the detailed analysis, we first introduce the incoherence parameter of a matrix, which is commonly used in many low-rank matrix applications. 
\begin{definition}[$\mu$-incoherence] \label{def:incoh}
Given a rank-$r$ matrix $\bm{X}\in\mathbb{R}^{n \times d}$. 
Let $\bm{X}=\bm{W}\bm{\Sigma}\bm{V}^\top$ be its compact singular value decomposition. $\bm{X}$ is $\mu$-incoherent if there exists a constant $\mu$ such that
\begin{equation*}
    \max_i \|\bm{e}_i^\top\bm{W}\|\leq\sqrt{\frac{\mu r}{n}} \qquad \textnormal{and} \qquad
    \max_i \|\bm{e}_i^\top\bm{V}\|\leq\sqrt{\frac{\mu r}{d}},
\end{equation*}
where $\bm{e}_i$ denotes the $i$-th canonical basis vector.
\end{definition}
The $\mu$-incoherence describes the 
correlation between the column/row spaces and the canonical basis vectors. The larger $\mu$ value implies a higher {\it overlapping}, which leads to a better chance of successful reconstruction from sparse row/column samples. We next use the following proposition to characterize the efficiency of sampling in both columns and rows. 
\begin{prop}
\label{lem:0}
Let $\bm{X}\in\mathbb{R}^{n\times d}$ be a rank-$r$, $\mu$-incoherent matrix. Without loss of generality, we assume $n\geq d$. Let $\bm{E}\in\mathbb{R}^{n\times d}$ be a noise matrix. By uniformly sampling  $\mathcal{O}(\mu r \log n)$ columns and rows from the noisy $\bm{X}+\bm{E}$, Skeleton approximation can construct a matrix $\hat{\bm{X}}$ such that, with probability at least $1-\mathcal{O}(n^{-2})$,  
\begin{align*}
    \|\bm{X} - \hat{\bm{X}}\| \le \mathcal{O}\left(\frac{ \|\bm{E}\|\sqrt{nd}}{\mu r \log n}\right).
\end{align*}
\end{prop}


A similar result, under a slightly different setting, can be found in \citep{cai2021rcur}. For the completeness of the paper, we provide the proof here.
\begin{proof}
We resolve the sampling strategy. We consider a clear rank-$r$ matrix $\bm{X}\in\mathbb{R}^{n\times d}$, i.e., no additive noise and the rank is exact. Without loss of generality, we assume $n\geq d$. 
Provided $\bm{X}$ is $\mu$-incoherent, by \citep[Theorem~1.1]{chiu2013sublinear}, Skeleton approximation recovers $\bm{X}$ exactly, i.e., 
\[
\bm{X}=\bm{C}\bm{U}\bm{R},
\]
with probability at least $1-\mathcal{O}(n^{-2})$
if we sample $\mathcal{O}(\mu r \log n)$ rows and columns uniformly to form the submatrices $\bm{C}$ and $\bm{R}$. 

Thirdly, we resolve the error bound estimation. For the noisy matrix $\bm{X}+\bm{E}$, we directly apply \citep[Corollary~4.3]{hamm2021perturbations}.
Thus, we have
\[
\|\bm{X}-\hat{\bm{C}}\hat{\bm{U}}\hat{\bm{R}}\| \leq \mathcal{O}\left(\sqrt{ \frac{nd}{l_C l_R}}\right) \|\bm{E}\|,
\]
where $\hat{\bm{C}}$ and $\hat{\bm{R}}$ are sampled from the noisy matrix, $\hat{\bm{U}}$ is the pseudo-inverse of $\hat{\bm{C}}$ and $\hat{\bm{R}}$'s intersection, and $l_C$ (resp.~$l_R$) is the number of columns (resp.~rows) being sampled in $\hat{\bm{C}}$ (resp.~$\hat{\bm{R}}$). 

Note that this error bound assumes good column/row sampling, i.e., the clear submatrices corresponding to $\hat{\bm{C}}$ and $\hat{\bm{R}}$ can recover $\bm{X}$ exactly. Therefore, by combining the above two results, we show the claim in Proposition~\ref{lem:0}.
\end{proof}

Several works (e.g., \citep{chiu2013sublinear,drineas2008relative}) have proposed explicit methods to construct $\hat{\bm{X}}$. Those methods require computing the pseudo-inverse, generally inefficient in deep learning settings. \citep{xiong2021nystromformer} uses an approximation of the pseudo-inverse in the symmetric matrix setting. It is still an open question whether the approximated pseudo-inverse also works for the general matrix in deep learning settings. On the other hand, in the transformer model, a good matrix approximation is not our primary goal, and we thus pursue a different way that only maintains sufficient information to pass through the network via \eqref{eq:cur_output}. 




\subsection{Smoother Component}

Based on the analysis of Skeleton Sketching, the matrix incoherence parameter $\mu$ plays a crucial role in determining the number of rows and columns to sample.  Decreasing in $\mu$ leads to a smaller sampling size. 
Furthermore, the $\mu$-incoherence condition implies that the ``energy'' of the matrix is evenly distributed over its entries, i.e.,  the matrix is ``smooth'' \citep{candes2009exact}. An illustration of the Smoother in Skeleton Attention part is shown in Figure~\ref{fig:last}.
We smooth the input token matrix to ensure the sampling in rows and columns containing more local and/or global information. Thus, sampling several rows and columns from the smoothed token matrix can be more effective than the samples from the original token matrix. In this subsection, we propose a novel smoother component to reduce the incoherence parameter without introducing excessive information loss.

\begin{figure}[h]
\centering
\includegraphics[width=0.4\textwidth]{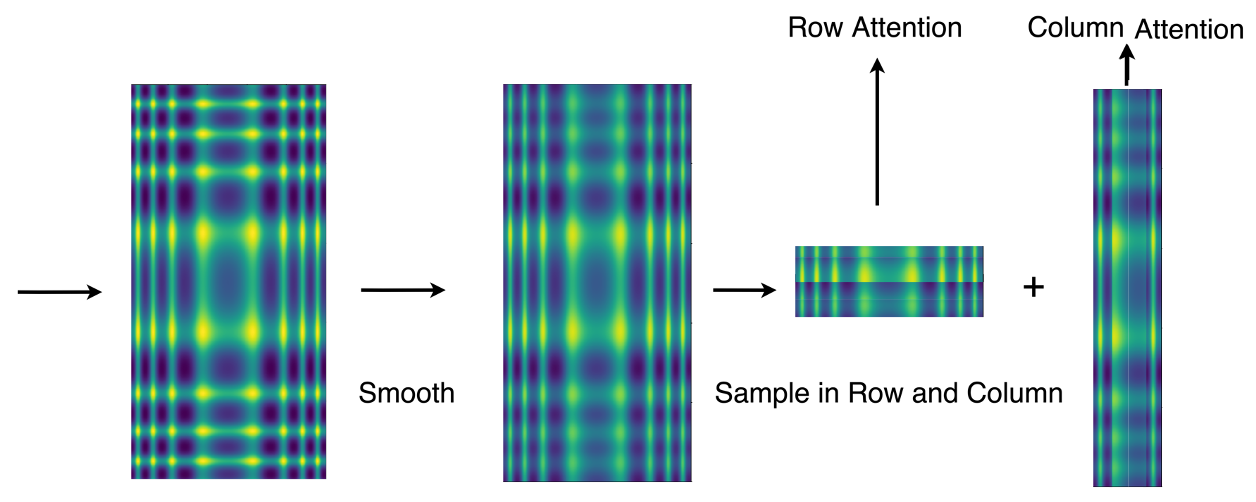}
\caption{Illustration on the effect of the Smoother in Skeleton Attention on Token Matrix.}
\label{fig:last}
\end{figure}

\vspace{0.05in}\noindent\textbf{Fourier Convolution.}~~ The incoherence parameter $\mu$ can be viewed as a measure of the smoothness of a matrix. A ``smoother'' matrix tends to have a smaller incoherence parameter. 
Intuitively, the adjacent columns or rows have similar values for a smooth matrix. Thus a few landmark columns or rows can represent the matrix with little error. On the other hand, if the matrix is harsh (e.g., containing spiky columns or rows), more landmarks are required. A common way to smooth a matrix is to convolute it with a smoothing kernel, such as a Gaussian kernel. However, directly using a fixed smoothing kernel can potentially remove too many details and harm the final performance. In the recent literature (e.g., \citep{guo2022visual}), large convolution kernel-based Attentions show a supreme performance in vision Transformers. In this paper, we propose to use a data-driven convolution layer along the sequence dimension with a kernel size equal to the sequence length. In this setting, the information of a given row could be decentralized among the rows. As the input token matrix is computed through a FeedForward layer, the information among different rows could be already processed. Hence, we do not perform the convolution along the hidden dimension. 

We use Fast Fourier Transformation (FFT) to implement the convolution. Let $\bm{L}_0\in\mathbbm{R}^{n\times  d}$ be the convolution kernel matrix. Via the convolution theorem,  the circular convolutions in the spatial domain are equivalent to pointwise products in the Fourier domain, and we then have:
\begin{align}
   \bm{X}^{\mathrm{smooth}} =  \bm{X}*\bm{L}_0 = \mathcal{F}^{-1}\left[\mathcal{F}(\bm{X})\cdot\mathcal{F}(\bm{L}_0) \right],\label{eq:fft_0}
\end{align}
where $\mathcal{F}$, $*$, and $\cdot$ denote FFT operator, convolution operator, and point-wise product, respectively.

Equation \eqref{eq:fft_0} requires $3d$ times faster Fourier operations which could be prohibited when facing large $d$. In order to save the computational cost, we use the learnable matrix $\bm{L}\in\mathbb{C}^{n\times d}$ in the frequency domain instead and apply segment-average (averaging segments of hidden dimension) to $\bm{X}$. In practice, we use the rFFT/irFFT, Fast (inverse) Fourier Transformation of real input instead of the general FFT/IFFT, and the size of the matrix $\bm{L}$ is reduced to $\bm{L}\in\mathbb{C}^{(\lfloor n/2\rfloor+1)\times d}$. To simplify the notation, we assume there are integers $s$ and $r$ with $d = sr$. Instead of using \eqref{eq:fft_0}, we apply the following \eqref{eq:fft_1} to smooth the token matrix. 
\begin{align}
     \bm{X}^{\mathrm{smooth}} = \mathcal{F}^{-1}\left[\mathcal{F}(\bm{X}\bm{S})\cdot\bm{L} \right],\label{eq:fft_1}
\end{align}
where 
\begin{align}
    \bm{S} = \begin{bmatrix}
    \frac{1}{s}\bm{1}&\bm{0}&...&\bm{0}\\
    \bm{0}&\frac{1}{s}\bm{1}&...&\bm{0}\\
    \vdots&\vdots&\ddots&\vdots\\
    \bm{0}&\bm{0}&\cdots&\frac{1}{s}\bm{1}
    \end{bmatrix}\in\mathbb{R}^{d\times d}\label{eq:s_DEF}
\end{align}
and $\bm{1}$ denotes the $s\times s$ matrix with all elements equal 1. As $\bm{X}\bm{S}$ contains repeated rows, in \eqref{eq:fft_1}, we can reduce the usage of faster Fourier operations to $r+d$ times. Moreover, to further control the computation cost, we only smooth the token matrix $\bm{X}$ instead of the query, key and value matrices $\bm{Q},\bm{K}$ and $\bm{V}$ separately.

In the following proposition, we show the smooth ability of the Fourier convolution.
\begin{prop}\label{lem:1}
Let $\{x_1,....,x_{n}\}$ be sequences with $\max_t |x_t|\le a_{\max}$ and $\max_{t}|x_{t}-x_{t-1}|\le b_{\max}$. Let $\{l_{1},...,l_{n}\}$ be sequences of \textit{i.i.d.}~$\frac{1}{n^2}\sigma^2$-subgaussian variables. Let $f(t)$ be the convolution of $\{x_t\}$ and $\{l_t\}$, i.e., $f(t) = \sum_{i=1}^{t}l_{t+1-i}x_i$. With probability at least $1-\delta$, we have:
\begin{align*}
    &\quad~|f(t)-f(t-1)|\cr
    &\le b_{\max}\sigma\sqrt{\frac{1}{2n}\log\left(\frac{2n}{\delta}\right)} + a_{\max}\sigma\sqrt{\frac{1}{2n^2}\log\left(\frac{2}{\delta}\right)}.
\end{align*}
\end{prop}

\begin{proof}

As $f(t)$ is the convolution function of $\{x_t\}$ and $\{l_t\}$, from the definition of convolution  for $t=1,2,...,n$ we have 
\begin{align}
f(t) & =\sum_{i=1}^{t}l_{t+1-i}x_i\notag
\end{align}
and
\begin{align}
     f(t)-f(t-1)&=\underbrace{\sum_{i=1}^{t-1}(l_{i+1}-l_i)x_i}_{:= (a_t)} + l_1x_{t}.\label{eq:lem:1:1}
\end{align}
By Hoffelding inequality, term $(a_t)$ satisfies the following inequality with $\varepsilon>0$.
\begin{align}
   \mathbbm{P}(|(a_t)|\ge \varepsilon) &= \mathbbm{P}\left( \left|\sum_{i=1}^{t-1}(l_{i+1}-l_i)x_i\right|\ge \varepsilon\right)\notag\\
   &\le \exp\left(-\frac{2\varepsilon^2}{(t-1)b_{\max}^2\cdot\frac{1}{n^2}\sigma^2}\right)\label{eq:lem:1:2}
\end{align}
Combine \eqref{eq:lem:1:2} with the union bound over $t=1,2,...,n$ and the following \eqref{eq:lem:1:2} holds with probability at least $1- \delta/2$: 
\begin{align}
    \max_t\left|(a_t)\right|\le b_{\max}\sigma\sqrt{\frac{1}{2n}\log\left(\frac{2n}{\delta}\right)}\label{eq:lem:1:3}
\end{align}

Similarly, with probability $1-\delta/2$, we have 
\begin{align}
    \max_{t}|l_1x_{t}|\le a_{\max}\sigma\sqrt{\frac{1}{2n^2}\log\left(\frac{2}{\delta}\right)}\label{eq:lem:1:4}.
\end{align}
Therefore, with probability at least $1-\delta$, \eqref{eq:lem:1:3} and \eqref{eq:lem:1:4} give
\begin{align*}
    \max_{t}&|f(t)-f(t-1)|\cr
    &\le b_{\max}\sigma\sqrt{\frac{1}{2n}\log\left(\frac{2n}{\delta}\right)} + a_{\max}\sigma\sqrt{\frac{1}{2n^2}\log\left(\frac{2}{\delta}\right)}.
\end{align*}
This finishes the proof.
\end{proof}
The Proposition~\ref{lem:1} can be used to describe the Fourier convolution layer's behavior in the early training stage. Via some standard initialization methods (e.g., Kaiming initialization or Xavier initialization), the variance of elements in learnable matrix $\bm{L}$ is $\mathcal{O}(n^{-1})$ and the scale of elements is $\mathcal{O}(n^{-1/2})$.\footnote{Here we omit the dependence in $d$ for brevity.} To simplify our discussion, let us assume we use Kaiming normal initialization and $\bm{L}$ becomes a random complex Gaussian matrix with zero mean and variance $n^{-1}\sigma^2$. Using the fact that the FFT of Gaussian s remains Gaussian with $2n$ times larger variance, the $n^{-1}\sigma^2$ variance Gaussian sequences through inverse FFT (IFFT) would result in Gaussian sequences with $\frac{1}{2n^2}\sigma^2$ variance. By Proposition~\ref{lem:1}, the maximum difference between adjacent elements after the convolution is scaled on $b_{\max}\sigma n^{-1/2}+a_{\max}\sigma n^{-1}\approx b_{\max}\sigma n^{-1/2}$ when sequence length $n$ is large enough. Thus as long as  $\sigma< \mathcal{O}(\sqrt{n})$, the sequences are smoothed by the Fourier convolution.  This smoothness ability eventually helps stabilize the early training stage. The related experiments are summarized in Section~\ref{appendix:M}.

When the training procedure progresses, the learnable matrix $\bm{L}$ will be apart from the Sub-Gaussian initialization and converge to some data-driven matrix. In this regime, it is possible that the convoluted sequences may not necessarily preserve a very low variance as it is in the early stage. However, a good memorization ability can happen. In particular, in Proposition~\ref{lem:2}, we show it is possible to have a learned matrix $\bm{L}$ such that the historical information in $\bm{X}\bm{S}$ up to time $t$ can be compressed to some function parameterized with $t$-th token in the convoluted sequences.


\begin{prop}\label{lem:2}
Let $\bm{X}\in\mathbbm{R}^{n\times d}$ be a bounded matrix and $\bm{S}\in\mathbbm{R}^{d\times d}$ constructed by \eqref{eq:s_DEF}. There exist matrices $\bm{G},\bm{L}\in\mathbbm{R}^{n\times d}$ such that
\begin{align*}
   \left\|(\bm{X}\bm{S})_{1:t} - \bm{X}^{\mathrm{smooth}}_{t}\bm{G}_{1:t}\right\|\le \mathcal{O}\left(r^{3/2}t\log(n) {d}^{-1/2}\right), 
\end{align*}
where $(\cdot)_{1:t}$ is the submatrix of the first $t$ rows of a given matrix, $\bm{X}^{\mathrm{smooth}}_{t}$ is the $t$-th row of $\bm{X}^{\mathrm{smooth}}= \mathcal{F}^{-1}\left[\mathcal{F}(\bm{X}\bm{S})\cdot\bm{L} \right]$, and $\bm{G}$ satisfies $\bm{G}_{i,j}=\bm{G}_{i+s,j}=....=\bm{G}_{i+r(s-1),j} = g_{i}(j)$. Here $\{g_1(\cdot),...,g_s(\cdot)\}$ is an orthogonal polynomial basis. 
\end{prop}

\begin{proof}

The proof contains two parts. In the first part, we view the data sequence as a function of index $t$ and  construct the coefficients and orthogonal polynomials for function approximation. In the second part, we show such coefficients can be computed with Fourier convolution, i.e., \eqref{eq:fft_1}).

\noindent{\it Function Approximation.}~~ We reformulate the matrix $\bm{X}\bm{S}$ as follow:
\begin{align}
   \bm{X}\bm{S} = \begin{bmatrix}
    \bar{\bm{x}}_1\bm{e}&\bar{\bm{x}}_2\bm{e}&\cdots&\bar{\bm{x}}_r\bm{e}\notag
    \end{bmatrix},
\end{align}
where $\bm{e}\in\mathbb{R}^{1\times s}$ is the one vector and $\bar{\bm{x}}_i\in\mathbb{R}^{n\times 1}$ is the average of $(s(i-1)+1)$-th to $(si)$-th columns of $\bm{X}$.

Next, we focus on vector $\bar{\bm{x}}_j$ and view its $t$-th element as the output of a function $h^j(t) =\bar{\bm{x}}_{jt}$. Via analysis in \citep[Appendixes~C and D]{gu2020hippo}, we can form an approximation on $h^j(t)$ as follow:
\begin{align}
    h_{[x\le t]}^j(x) \approx \sum_{i=1}c_i^j(t) g_i(x),
\end{align}
where $\{g_i\}$ is a sequence of orthogonal polynomial and $[c_1^j(t),c_2^j(t),...,c_s^j(t)] := \bm{c}_t^j\in\mathbbm{R}^{1\times s}$ satisfy
\begin{align}
    \frac{d}{dt}\bm{c}(t)^{j} = \frac{1}{t}\bm{c}(t)^{j}\bm{A}_0+ \frac{1}{ts\log n }h(t)\bm{b}_0\label{eq:lem:2:1}
\end{align}
where $\bm{A}_0\in\mathbbm{R}^{s\times s}$ and $\bm{b}_0\in\mathbbm{R}^{1\times s}$ are predefined matrix and vector respectively. Equation \eqref{eq:lem:2:1} is corresponding to the case with $\lambda_n = s\log n$ in \citep{gu2020hippo}. 

We then use Forward Euler approach to discretize it: 
\begin{align}
    \hat{\bm{c}}(t)^j = \hat{\bm{c}}(t-1)^j(\frac{1}{t}\bm{I} +\frac{1}{t}\bm{A}_0)  + \frac{1}{ts\log n}h(t)\bm{b}_0,\label{eq:1}.
\end{align}
The standard error analysis of Forward Euler approach gives
\begin{align}
   &\bm{c}(t+1)^j \notag\\
   =& \bm{c}(t)^j + \frac{1}{t}\bm{c}(t)^j\bm{A}_0 + \frac{1}{ts\log n}h(t)\bm{b}_0 + \frac{d^2}{dt^2}\bm{c}(t)^j|_{t = \xi}
   \notag\\
   =&\bm{c}(t)^j + \frac{1}{t}\bm{c}(t)^j\bm{A}_0 + \frac{1}{ts\log n}h(t)\bm{b}_0 + \frac{1}{\xi  s\log n}h(\xi)^{\prime}\bm{b}_0\notag\\
   =&\bm{c}(t)^j + \frac{1}{t}\bm{c}(t)^j\bm{A}_0 + \frac{1}{ts\log n}h(t)\bm{b}_0 + \mathcal{O}\left(\frac{1}{ts\log n}\right),\notag
\end{align}
where $\xi\in[t,t+1]$. It implies that for $t =1,2,...,n$,
\begin{align}
    \|\hat{\bm{c}}(t)^j - \bm{c}(t)^j\|\le \mathcal{O}\left(\frac{\log t}{s\log n}\right).\label{eq:euler_error}
\end{align}                                                                                                                                                 
Combine \eqref{eq:euler_error} with Proposition~6 in \citep{gu2020hippo}, if $h^{j}(x)$ is quadratic spline interpolation on $\{\bar{\bm{x}}_{jt}\}$, we obtain\footnote{If $h^{j}$ is $(k+1)$-th order spline interpolation on $\{\bar{\bm{x}}_{jt}\}$, then$    \|\bar{\bm{x}}_{jt} - \sum_{i=1}^s\hat{\bm{c}}_{i}(t)g_i(x)\|\le\mathcal{O}\left(t^{k}\log^k n \left(\frac{r}{d}\right)^{-k+1/2}\right)$. },
\begin{align}
    \|\bar{\bm{x}}_{jt} - \sum_{i=1}^s\hat{\bm{c}}_{i}(t)g_i(x)\|
    &\le \mathcal{O}\left(t\log n/\sqrt{s}\right) \cr
    &=  \mathcal{O}\left(t\log n\sqrt{\frac{r}{d}}\right).
    \label{eq:lem:2:2}
\end{align}
The desirable result in Proposition~\ref{lem:2} is obtained by repeatedly using \eqref{eq:lem:2:2} with $j=1,2,...,r$.\\

\noindent{\it Coefficients via Fourier Convolution.}~~ The remaining task is to show that $\{\hat{\bm{c}}(t)^{j}\}$ can be generated via Fourier convolution. To simplify the notation, we denote $\bm{A} = \frac{1}{t}\bm{I}+\frac{1}{t}\bm{A}_0$ and $\bm{b} = \frac{1}{t\log n}\bm{b}_0$ and \eqref{eq:1} becomes
\begin{align}
    \hat{\bm{c}}(t)^j =  \hat{\bm{c}}(t-1)^j\bm{A} + h(t)\bm{b}\label{eq:2}.
\end{align}
We then repeatedly use \eqref{eq:2} with $t = 1,2,...$, and one can see
\begin{align}
    &\hat{\bm{c}}_t^j = \sum_{i=1}^{t-1}\bm{b}\bm{A}^{t-i}h(i) = \sum_{i=1}^{t-1}\bm{b}\bm{A}^{t-i}\bar{\bm{x}}_{ji}\notag\\
    \Longrightarrow\quad& \bm{C}^{j} = \bar{\bm{A}}_j * (\bar{\bm{x}}_j\bm{e}),\label{eq:conv:1}
\end{align}
where 
\begin{align}
    \bm{C}^{j} = \begin{bmatrix}
    \hat{\bm{c}}_1^j\\
    \hat{\bm{c}}_2^j\\
    \vdots\\
    \hat{\bm{c}}_n^j
    \end{bmatrix} \in\mathbb{R}^{n\times s},\ \textnormal{ and }\ \bar{\bm{A}}_j = \begin{bmatrix}
    \bm{b}\\
    \bm{b}\bm{A}\\
    \vdots\\
    \bm{b}\bm{A}^{n-1}
    \end{bmatrix}  \in\mathbb{R}^{n\times s}.
\end{align}

Next we repeatedly use \eqref{eq:conv:1} from $j=1,2,..,r$, and have
\begin{align}
   &\underbrace{ \begin{bmatrix}
    \bm{C}^1
    &\bm{C}^2
    &\cdots
    &\bm{C}^r
    \end{bmatrix}}_{:=\bm{X}^{\mathrm{smooth}}}\notag\\
    =& 
   \underbrace{\begin{bmatrix}
    \bar{\bm{A}}_1
    &\bar{\bm{A}}_2
    &\cdots
    &\bar{\bm{A}}_r
    \end{bmatrix}}_{:=\bm{L}_0} * 
  \underbrace{\begin{bmatrix}
   \bar{\bm{x}}_1 \bm{e}
    &\bar{\bm{x}}_2 \bm{e}
    &\cdots
    &\bar{\bm{x}}_r \bm{e}
    \end{bmatrix}}_{=\bm{X}\bm{S}}\notag\\
    \Longrightarrow\quad & \bm{X}^{\mathrm{smooth}} = \bm{L}_0 * \bm{X}\bm{S}\notag\\
    \Longrightarrow\quad & \bm{X}^{\mathrm{smooth}} = \mathcal{F}^{-1}\left(\mathcal{F}(\bm{L}_0) \cdot \mathcal{F}(\bm{X}\bm{S})\right)\notag\\
    \Longrightarrow\quad & \bm{X}^{\mathrm{smooth}} = \mathcal{F}^{-1}\left(\bm{L} \cdot \mathcal{F}(\bm{X}\bm{S})\right),\notag
\end{align}
where  we use the fact that $\bm{L}$ is constructed in frequency domain in Fourier convolution in Eq. \eqref{eq:fft_1}.
\end{proof}

The Proposition~\ref{lem:2} states that if we properly train the matrix $\bm{L}$, the information in $\bm{X}\bm{S}$ up to row $t$ can be compressed into $t$-th row of $\bm{X}^{\mathrm{smooth}}$ with a moderate tolerance. Therefore, when we sample in rows $\bm{X}^{\mathrm{smooth}}$, they will contain more information than the same number of rows in the original $\bm{X}\bm{S}$. Similar results are also discussed in FNet \citep{lee2021fnet} and several RNN literature, such as \citep{gu2020hippo} and \citep{voelker2019legendre}. In  \citep{gu2020hippo},  several specific types of polynomials (e.g., Legendre or Chebyshev) are explored, and the corresponding matrix $\bm{L}$ is predefined instead of data-driven. Recently, \citep{gu2021efficiently} propose a sophisticated method that can be used to compute $\bm{X}^{\mathrm{smooth}}$. We leave it for future work.

\vspace{0.05in}\noindent\textbf{Convolution Stem.}~~ $\bm{X}^{\mathrm{smooth}}$ may encounter an over-smoothing situation that local details can be wiped out. We use a convolution stem (CNNs + BN + ReLU) to tackle this problem. We first concatenate $\bm{X}^{\mathrm{smooth}}$ with the original token matrix $\bm{X}$ into a $n\times 2d$ matrix and then apply a 1D convolution with kernel size 3 to transform it back to $n\times d$ dimensions. At last, the output is normalized with the Batchnorm layer and truncated by the ReLU activation function to stabilize the training procedure. \citep{wang2021scaled} report the ReLU activation coupled with the normalization layer plays an important role in various vision transformers and analyzes this phenomenon theoretically.

\section{Experiments} \label{sec:experiment}
In this section, we test our S$^3$Attention on Long Range Arena (LRA) datasets \citep{tay2021long} and six real-world time series benchmark datasets for long-term forecasting. We also evaluate the transfer learning ability of S$^3$Attention on GLUE tasks. In recent literature (e.g., \citep{gu2020hippo,gu2021efficiently,gupta2022diagonal}), the RNN type model is also widely discussed for long sequence tasks. We don't include them as benchmark models since in this paper we focus on improving the Attention structure. The testing environment contains 12 Intel(R) Xeon(R) Platinum 8163 CPU @ 2.50GHz CPUs, 1 TESLA V100 SXM2 32G, and 90 GB memory. We implement the S$^3$Attention based on the official codes of \citep{zhu2021long} and \citep{FedFormer} for LRA and time-series forecasting tasks respectively. The implementation details for S$^3$Attention are provided in the supplementary material and the code is available online at \url{https://github.com/wxie9/S3Attention}.

\subsection {Long-Range Arena}
The open-source Long-Range Arena (LRA) benchmark \citep{tay2021long} is originally proposed as a standard way to test the capabilities of Attention variants architectures on long sequence tasks.

 We benchmark our model with several recent state-of-art efficient Attention architectures, including Sparse Transformer {\citep{child2019generating}}, Longformer  {\citep{beltagy2020longformer}}, Linformer {\citep{wang2020linformer}}, Reformer {\citep{kitaev2020reformer}}, Sinkhorn Transformer {\citep{tay2020sparse}}, Synthesizer {\citep{tay2005synthesizer}}, BigBird  {\citep{zaheer2020big}},  Linear Transformers {\citep{katharopoulos2020transformers}}, Performer {\citep{choromanski2020rethinking}},  H-Transformer-1D {\citep{zhu2021h}}, Nystr{\"o}mformer {\citep{xiong2021nystromformer}}, Transformer-LS {\citep{zhu2021long}}, FNet {\citep{lee2021fnet}}, Luna {\citep{ma2021luna}}, FMMformer {\citep{nguyen2021fmmformer}},  Cosformer {\citep{qin2021cosformer}} and Scatterbrain {\citep{chen2021scatterbrain}}. S$^3$Attention achieves the highest 66.08\% average accuracy with tuned parameters and the second best 64.11\% result with fixed parameters as shown in Table~\ref{tab:lra-benchmarks}.

In particular, S$^3$Attention significantly outperforms the benchmarks on Image tasks by relatively large margins (12.6\% and 20.6\%, respectively), which support S$^3$Attention's smoothness effect on the low-level features and will benefit the high-level image classification tasks.

Moreover, we highlight the sampling efficiency of S$^3$Attention. The sequence length of LRA tasks is over one thousand. The efficient Transformers in literature usually can not project the token matrix to a very small size while maintaining comparable numerical performance, by only sampling 8 rows and columns from the token matrix, S$^3$Attention has already obtained 64.11\% average score improving the previous best 62.03\% score of Transformer-LS.
\begin{table*}[tbp]
\centering
\caption{Experimental results on Long-Range Arena benchmark. The best model is in boldface and the second best is underlined. The standard deviations of the S$^3$Attention are reported in parentheses.}
\scalebox{1}{
\begin{tabular}{l|ccccc|c}
\toprule
Model& ListOps& Text & Retrieval & Image & Pathfinder& Average\\
\midrule
Transformer          &{36.37}    &{64.27}     &{57.46}      &{42.44}      &{71.40}|   &{54.39}\\

Local Attention      &15.82    &52.98     &53.39      &41.46      &66.63   &46.06\\
Sparse Transformer   &17.07    &63.58     &59.59      &44.24      &71.71   &51.24\\
Longformer           &35.63    &62.85     &56.89      &42.22      &69.71   &53.46\\
Linformer            &35.70    &53.94     &52.27      &38.56      &76.34   &51.36\\
Reformer             &37.27    &56.10     &53.40      &38.07      &68.50   &50.67\\
Sinkhorn Transformer &33.67    &61.20     &53.83      &41.23      &67.45   &51.39\\
Synthesizer          &36.99    &61.68     &54.67      &41.61      &69.45   &52.88\\
BigBird              &36.05    &64.02     &59.29      &40.83      &74.87   &55.01\\
Linear Transformer   &16.13    &65.90     &53.09      &42.34      &75.30   &50.55\\
Performer            &18.01    &65.40     &53.82      &42.77      &77.05   &51.41\\
Nystromformer        &37.34    &65.75     &81.29      &41.58      &70.94   &59.38\\
H-Transformer-1D     &{\bf49.53}    &{\bf 78.69}     &63.99      &46.05      &68.78   &61.41\\
Transformer-LS       &38.36    &68.40     &81.85      &45.05      &76.48   &62.03\\
FNet                 &35.33    &65.11     &59.61      &38.67      &77.08   &54.42\\
Luna                 &38.01    &65.78     &79.56      &47.86      &{\bf 78.89}   &62.02\\
FMMformer            &36.74    &67.84     &81.88      &45.10      &72.12   &60.74\\
PoNet                &38.80    &69.82     &80.35      &46.88      &70.39   &61.05\\
Cosformer &37.9&63.41&61.36&43.17&70.33&55.23\\
Scatterbrain &38.6&64.55&80.22&43.65&69.91&59.38\\
\midrule
S$^3$Attention ($r,s_1,s_2= 8$) & 38.30(0.40)&69.27(0.83)&\underline{83.26(0.45)}  &\underline{53.90(1.54)} &75.82(0.97)&\underline{64.11(2.07)}\\ 
S$^3$Attention (best)  & \underline{39.15(0.48)}& \underline{71.58(0.95)}&{\bf83.73(0.61)}  &{\bf57.73(1.83)} &\underline{78.20(1.32)}&{\bf66.08(2.56)}\\ 
\bottomrule
\end{tabular}
}
\label{tab:lra-benchmarks}
\end{table*}
\subsection{Long-Term Forecasting Tasks for Time Series }\label{sec:forecasting}
To further evaluate the proposed S$^3$Attention, we also conduct extensive experiments on six popular real-world benchmark datasets for long-term time series forecasting, including traffic, energy, economics, weather, and disease as shown in Table \ref{tab:multi-benchmarks}

To highlight the relevant comparison, we include five state-of-the-art (SOTA) Attention-based models, i.e., FEDformer~\citep{FedFormer}, Autoformer~\citep{Autoformer}, Informer~\citep{haoyietal-informer-2021}, LogTrans~\citep{Log-transformer-shiyang-2019}, and Reformer~\citep{kitaev2020reformer} for comparison. FEDformer is selected as the main baseline as it achieves SOTA results in most settings. More details about baseline models, datasets, and implementations are described in the supplementary material. 

Compared with SOTA work (i.e., FEDformer), the proposed S$^3$Attention yields a comparable performance in those tasks, with {3/6 datasets having larger winning counts in MSE/MAE}. It is worth noting that the improvement is even more significant on certain datasets, e.g., Exchange ($>30\%$ reduction in MSE and MAE). Although Exchange does not exhibit an apparent periodicity pattern, S$^3$Attention still achieves superior performance.  

\begin{table*}[tbp]
\centering
\caption{Multivariate long-term series forecasting results on six datasets with input length of $96$ and prediction length $O \in \{96,192,336,720\}$ (For ILI dataset, we set prediction length $O \in \{24,36,48,60\}$) with input length {$36$}. A lower MSE indicates better performance. All experiments are repeated 5 times.}
\scalebox{1}{
\begin{tabular}{c|c|cccccccccccccccccc}
\toprule
\multicolumn{2}{c|}{Methods}&\multicolumn{2}{c|}{S$^3$Attention}&\multicolumn{2}{c|}{FEDformer}&\multicolumn{2}{c|}{Autoformer}&\multicolumn{2}{c|}{Informer}&\multicolumn{2}{c|}{LogTrans}&\multicolumn{2}{c}{Reformer}\\
\midrule
\multicolumn{2}{c|}{Metric} & MSE  & MAE & MSE & MAE& MSE  & MAE& MSE  & MAE & MSE  & MAE& MSE  & MAE\\
\midrule
\multirow{4}{*}{\rotatebox{90}{ETTm2}} &96 & \textbf{0.192} & \textbf{0.283} &0.203 &0.287 &0.255  &0.339 &0.705 &0.690   &0.768  &0.642  &0.658  &0.619    \\
                        & 192 & \textbf{0.255} & \textbf{0.324} &0.269  &0.328  &0.281 &0.340 &0.924 &0.692   &0.989  &0.757  &1.078  &0.827    \\
                        & 336 & \textbf{0.324} & \textbf{0.364} & 0.325 &0.366 &0.339  &0.372 &1.364 &0.877  &1.334  &0.872  &1.549  &0.972     \\
                        & 720 & 0.431 & 0.433 &\textbf{0.421} &\textbf{0.415} &0.422  &0.419 &0.877 &1.074  & 3.048 &1.328  &2.631  &1.242      \\
\midrule
\multirow{4}{*}{\rotatebox{90}{Electricity}} &96  &0.218  &0.332  &\textbf{0.183} &\textbf{0.297} &0.201  &0.317 &0.304 &0.405   &0.258  &0.357  &0.312  &0.402    \\
                        & 192 &0.259 & 0.361 & \textbf{0.195} &\textbf{0.308} &0.222  &0.334 &0.313 &0.413   &0.266 &0.368  &0.348  &0.433    \\
                        & 336 & 0.267 & 0.367 &\textbf{0.212} &\textbf{0.313} &0.231 &0.338 &0.290 &0.381   &0.280 &0.380  &0.350  & 0.433    \\
                        & 720 & 0.293 & 0.385&\textbf{0.231} &\textbf{0.343} &0.254  &0.361 &0.262 &0.344  &0.283  &0.376  &0.340  &0.420     \\
\midrule
\multirow{4}{*}{\rotatebox{90}{Exchange}} &96  & \textbf{0.086} & \textbf{0.204} &0.139 &0.276 &0.197  &0.323 &1.292 &0.849   &0.968  &0.812  &1.065  &0.829    \\
                        & 192 & \textbf{0.188} & \textbf{0.292} &0.256 &0.369 &0.300  &0.369 &1.631 &0.968  &1.040  &0.851  &1.188  & 0.906   \\
                        & 336 & \textbf{0.356} & \textbf{0.433} &0.426 &0.464 &0.509  &0.524 &2.225 &1.145   &1.659  &1.081  &1.357  &0.976     \\
                        & 720 & \textbf{0.727} & \textbf{0.669} &1.090 &0.800 &1.447  &0.941 &2.521 &1.245   &1.941  &1.127  &1.510  &1.016     \\
\midrule
\multirow{4}{*}{\rotatebox{90}{Traffic}} &96  &0.592& 0.352 &\textbf{0.562} &\textbf{0.349} &0.613  &0.388 &0.824 &0.514  &0.684  &0.384  &0.732  &0.423    \\
                        & 192 & 0.583 & 0.343 &\textbf{0.562} &\textbf{0.346} &0.616&0.382 &1.106 &0.672   &0.685  &0.390  &0.733  &0.420    \\
                        & 336 & 0.598 & 0.346 &\textbf{0.570} &\textbf{0.323} &0.622  &0.337 &1.084 &0.627   &0.733  &0.408  &0.742  &0.420     \\
                        & 720 & 0.641 &0.397 &\textbf{0.596} &\textbf{0.368} &0.660  &0.408 &1.536 &0.845   &0.717  &0.396  &0.755  &0.423     \\
\midrule
\multirow{4}{*}{\rotatebox{90}{Weather}} & 96 & \textbf{0.182} & \textbf{0.262} &0.217  &0.296  &0.266  &0.336 &0.406 &0.444   &0.458  &0.490  &0.689  &0.596    \\
                        & 192 & \textbf{0.228} & \textbf{0.306} &0.276  &0.336  &0.307  &0.367 &0.525 &0.527   &0.658  &0.589  &0.752  &0.638    \\
                        & 336 &\textbf{0.295} &\textbf{0.355} & 0.339  &0.380  &0.359  &0.395 &0.531 &0.539  &0.797  &0.652  &0.639  &0.596    \\
                        & 720 &\textbf{0.383} & \textbf{0.418} &0.403  &0.428 &0.578 &0.578 &0.419  &0.428    &0.869  &0.675  &1.130  &0.792    \\
\midrule
\multirow{4}{*}{\rotatebox{90}{ILI}} & 24 & {2.431}  &{0.997} &\textbf{2.203}  &\textbf{0.963}  &3.483 &1.287 &4.631 &1.484   &4.480  &1.444  &4.400 &1.382    \\
                        & 36 & {2.287}& \textbf{0.972} &\textbf{2.272}  &0.976  &3.103  &1.148 &4.123 &1.348   &4.799  &1.467  &4.783  &1.448    \\
                        & 48 & {2.418} & {1.002} &\textbf{2.209}  &\textbf{0.981}  &2.669  &1.085 &4.066 &1.36   &4.800  &1.468  &4.832  &1.465    \\
                        & 60 & \textbf{2.425} & \textbf{1.043} &2.545  &1.061  &2.770  &1.125 &4.278 &1.41   &5.278  &1.560  &4.882  &1.483    \\
\midrule
 & 1st Count &12 & 13&12&11&0&0&0&0&0&0&0&0\\
\bottomrule
\end{tabular}
\label{tab:multi-benchmarks}
}
\end{table*}
\begin{table*}[h]
\centering
\caption{GLUE validation results. We report the mean of accuracy and F1 for QQP and MRPC, matthew correlations for CoLA, spearman correlations for STS-B, and accuracy for other tasks. For MNLI task, we consider the matched test set.}
\scalebox{1}{
\begin{tabular}{l|cccccccc|c}
\toprule
Model&MNLI&QQP&QNLI&SST-2&CoLA&STS-B&MRPC&RTE&Average\\
\midrule
BERT-Base &81.98	&89.25	&88.22	&91.07	&48.08	&87.98	&86.43	&69.98	&80.37\\
FNet-Base &73.20	&85.83	&80.57	&88.66	&40.67	&80.64	&80.88	&57.41	&73.48\\
PoNet-Base&77.02	&87.59	&84.37	&89.29	&45.38	&84.66	&81.82	&64.27	&76.80\\
S$^3$Attention (Ours)&76.86	&87.67	&84.12	&90.14	&46.72	&84.87	&81.84	&63.87	&77.01\\
\bottomrule
\end{tabular}
}
\label{tab:Glue}
\end{table*}
\begin{table*}[htbp]
\centering
\caption{Benchmark results of all Xformer models with a consistent batch size of 32 across all models with various input lengths { on the LRA text classification task The speed-up and memory-saving multipliers relative to Transformer shown in parentheses.}}
\scalebox{1}{
\begin{tabular}{l|cccc|cccc}
\toprule
& \multicolumn{4}{|c|}{ Training Speed (Steps per second) } & \multicolumn{4}{|c}{ Peak Memory Usage (GB) } \\
Model & $1 \mathrm{~K}$ & $2 \mathrm{~K}$ & $3 \mathrm{~K}$ & $4 \mathrm{~K}$ & $1 \mathrm{~K}$ & $2 \mathrm{~K}$ & $3 \mathrm{~K}$ & $4 \mathrm{~K}$ \\
\midrule
Transformer & $23.8$ & $7.8$ & $3.9$ & $OOM$ & $3.7$ & $11.1$ & $22.1$ & $OOM$ \\
Linformer & $37.0(1.5 \mathrm{x})$ & $20.8(2.6 \mathrm{x})$ & $14.9 ( 3.7 \mathrm { x } )$ & $11.9$ & $2 . 3$ & $3 . 3$ & $4.3$ & $5 . 2$ \\
Reformer & $28.5(1.2 \mathrm{x})$ & $15.1(1.9 \mathrm{x})$ & $11.9(3.0 \mathrm{x})$ & $9.1$ & $2.2$ & $3.2$ & $4.2$ & $4.9$ \\
Nystroformer & $33.3(1.4 \mathrm{x})$ & $22.7(2.9 \mathrm{x})$ & $17.2(4.3 \mathrm{x})$ & $14.7$ & $1.6$ & $2.2$ & $2.4$ & $2.9$ \\
Performer & $29.4(1.2 \mathrm{x})$ & $16. 9(1.9 \mathrm{x})$ & $8 . 7(11.7 \mathrm{x})$ & $9. 3$ & $2.3$ & $3.1$ & $4.0$ & $4.8$ \\
S$^3$Attention & $32.2(1.4 \mathrm{x})$ & $20.4(2.6 \mathrm{x})$ & $15.9(4.0 \mathrm{x})$ & $12.3$ & $1.8$ & $2.3$ & $2.8$ & $3.2$ \\
\bottomrule
\end{tabular}
}
\label{tab:speed_memory}
\end{table*}

\subsection{Transfer Learning in GLUE Tasks}
We evaluate the transfer learning ability of the proposed model in the pretraining-finetuning paradigm in NLP tasks. We pre-train vanilla BERT \citep{devlin2018bert}, FNet \citep{lee2021fnet}, PoNet \citep{tan2021ponet} and our S$^3$Attention with the same MLM loss in \citep{devlin2018bert} on English Wikitext-103 and BooksCorpus datasets. All models are uncased and pre-trained with the same configuration with 1 million steps at most. We report the best GLUE results for each model from multiple hyper-parameters configurations in Table \ref{tab:Glue}, and the detailed training configurations in Table~\ref{tab:BERT configure} in Table~\ref{tab:Glue}. Our S$^3$Attention reaches 77.01 average scores (\textbf{96.0\%} of the accuracy of vanilla BERT), which also outperform FNet by \textbf{4.6\%} and PoNet by 0.3\% relatively.

\subsection{Training Speed and Peak Memory Usage}
We compared the training speed (in terms of steps per second) and peak memory usage with several baseline models {in LRA text classification task with various input lengths. The results are reported in Table~\ref{tab:speed_memory}}. S$^3$Attention achieves a 4x time speed advantage and 87\% memory reduction compared to vanilla transformer models with 3k input setting and has a neck-to-neck performance compared to the most efficient baseline models.


\subsection{Robustness Analysis}
We conduct a noise-resistant experiment for S$^3$Attention and 5 other Attention based models as shown in Table~\ref{tab:robustness}. We use the Image experiment setting in LRA datasets. During generating sample sequences, we randomly add noise with uniform distribution $\mathcal{U}(-a,a)$ to each position in the sequences. We consider $a \in [0,2,4,8]$ and train every model with 5k steps and 5 replicates. S$^3$Attention remains robust with a high level of noise injection. This supports our theoretical robustness analysis and shows S$^3$Attention indeed makes an appropriate tradeoff between information preservation and noise reduction.

\begin{table*}
\centering
\caption{ Average Accuracy on Image task (CIFAR-10 dataset) in Long Range Arena with noise injections. The relative performance changes are reported in parentheses. }
\scalebox{1}{
\begin{tabular}{l|cccc}
\toprule
Noise level& 0 & 2 & 4 & 8 \\
\midrule
Transformer &41.39&	40.29 (-2.82\%)	&28.56 (-31.12\%)	&	28.12 (-32.18\%)	\\	
Linformer &38.43&	37.99 (-1.49\%)	&	37.04 (-3.95\%)	&	36.65 (-4.97\%)	\\	
Reformer &38.04	&37.64 (-1.12\%)	&	35.26 (-7.37\%)	&	34.88	 (-8.37\%)\\	
Nystroformer& 41.52	&40.89 (-1.66\%)&		38.39 (-7.67\%)	&	37.84 (-8.99\%)	\\	
Performer &42.66&	41.95 (-1.93\%)	&	39.61 (-7.40\%)	&	38.86 (-9.15\%)	\\	
S$^3$Attention& 57.47&	57.06 (-0.82\%)	&	55.32 (-3.84\%)	&	54.70 (-4.92\%)	\\	
\bottomrule
\end{tabular}
}
\label{tab:robustness}
\end{table*}




\subsection{Model Parameters Impact}
S$^3$Attention introduces three extra hyperparameters, $r, s_1$ and $s_2$. We test the influence when varying them and report results in Table~\ref{tab:impact}. We use S$^3$Attention $(r,s_1,s_2 = 8)$ as the baseline model and other parameters are reported in Table~\ref{tab:impact_params} in the supplementary material.  

\vspace{0.05in}\noindent{\bf Influence of $r$ in Fourier Convolution.}~~  The $r$ parameter is used to determine the number of segment-averages to compute in \eqref{eq:fft_1}. The smaller $r$ leads the matrix with more duplicate columns, and more details information is lost. On the other hand, according to Proposition~\ref{lem:2}, the larger $r$ would potentially decrease the memorization ability and yield a high approximation error. In Table~\ref{tab:varying r}, the best performance is observed when $r = 8$  or $r=16$.  For the case with $r = 1$, the token matrix is smoothed to rank one matrix, and the average accuracy drops 3.55 from the best setting. When the $r$ value goes larger than $16$, the accuracy in all experiments slightly decreases. We believe it is due to the over-fitting since the smoothed token matrix contains more flexibility and more irrelevant information training dataset is learned.

\vspace{0.05in}\noindent{\bf Influence of Sample Number $s_1$ in Row Attention.}~~ In Row Attention part, we randomly sample $s_1$ from key and value tokens. Table~\ref{tab:varying s_1} reports that the optimal sampling amounts are different among tasks. In Pathfinder task, the optimal result is associated with $s_1 = 256$, while the best performance of other tasks the reached with $s_1 = 32$. {Pathfinder} task requires learning extreme long-range dependence (the connectivity between two circles far away from each other). The lack of enough tokens leads to inaccurate long-range dependence estimation and damages the final results. For tasks like {Image} or { Retrieval}, the modest range dependence may already be enough to get promising performance, and we thus could use fewer token samples. 

\vspace{0.05in}\noindent{\bf Influence of Sample Number $s_2$ in Column Attention.}~~ In Column Attention,  $s_2$ columns are selected. The experiment results are shown in Table~\ref{tab:varying s_2}. When setting $s_2= 1$, average performance decreases by 13.24\%. Similar behavior is also observed in the first row of Table~\ref{tab:varying r} with $r = 1$. The information loss due to lack of rank limits the final performance. In an average sense, $s_2 = 16$ gives the best result, and further increasing in $s_2$ slightly harms the accuracy in all tasks except Pathfinder.


\begin{table}
\caption{Experimental results on varying $r,s_1$ and $s_2$. The best result is in boldface and the second best is underlined. And Ablation experiments for each component.}
\label{tab:impact}
\begin{subtable}{\linewidth}

\centering
\caption{Experimental results on varying $r$ parameter in smoothing component.}
\scalebox{1}{
\begin{tabular}{l|ccccc|c}
\toprule
$r$& LisOps& Text & Retrieval & Image & Pathfinder& Average\\
\midrule
$1$ & 37.30             &65.25              &78.65              &51.36              &71.23              &60.76\\ 
8 & \underline{38.30}   &\underline{69.27}  &{\bf83.26}         &\underline{53.90}  &\underline{75.82}         &\underline{64.11}\\ 
16 & {\bf 38.62}        &{\bf70.02}         &\underline{83.21}  &{\bf54.20}         &{\bf76.15}  &{\bf64.44}\\ 
32 & 38.19              &69.27              &82.05              &53.73              &75.58              &63.76\\ 
64 & 37.89              &69.73              &81.79              &51.28              &75.52              &63.24\\ 
\bottomrule
\end{tabular}
\label{tab:varying r}
}
\end{subtable}

\vspace{0.12in}
\begin{subtable}{\linewidth}
\centering
\caption{Experimental results on varying $s_1$ parameter in Row Attention.} \label{tab:varying s_1}
\scalebox{1}{
\begin{tabular}{l|ccccc|c}
\toprule
$s_1$& LisOps& Text & Retrieval & Image & Pathfinder& Average\\
\midrule
8   & \underline{38.30} &69.27   &\underline{83.26}  &\underline{53.90} &75.82   &64.11\\ 
32  & {\bf 38.44} &{\bf 70.85}   &{\bf 83.41}  &\bf{54.92} &77.97   &{\bf 65.12}\\ 
64  & 37.88 &\underline{70.53}   &83.02  &51.22 &78.02   &\underline{64.33}\\ 
128 & 37.33 &69.24   &81.58  &49.08 &\underline{78.12}   &63.07\\ 
256 & 37.02 &65.72   &79.30  &46.24 &{\bf78.14}   &61.29\\ 
\bottomrule
\end{tabular}
}
\end{subtable}

\vspace{0.12in}
\begin{subtable}{\linewidth}
\centering
\caption{Experimental results on varying $s_2$ parameter in Column Attention. }
\scalebox{1}{
\begin{tabular}{l|ccccc|c}
\toprule
$s_2$& LisOps& Text & Retrieval & Image & Pathfinder& Average\\
\midrule
$1$     &37.32	            &55.28	            &57.37	            &40.97	            &66.25	            &51.44\\
4       &\underline{37.82}	&52.05          	&72.58	            &46.74	            &73.17	            &57.47\\
8       &\bf{38.30}	        &\underline{69.27}	&\underline{83.26}	&\underline{53.90}  &75.82              &\underline{64.11}\\
16      &{37.77}	        &{\bf70.24}	        &{\bf 83.42}	    &{\bf 54.11}	    &\underline{77.92}  &{\bf64.73}\\
32      &37.62	            &68.32	            &80.11	            &51.66	            &{\bf78.18}	        &62.98\\
\bottomrule
\end{tabular}
\label{tab:varying s_2}
}
\end{subtable}
\hfill
\end{table}

\subsection{Learning Curve for LRA Experiments}\label{appendix:M}

In this section, we present the training and testing performance for the first 5000 training steps on five LRA datasets. All hyperparameters are kept the same as the baseline model in the ablation study. The average training accuracy/loss and test accuracy/loss are reported in Figure~\ref{fig:learning_curve}. The results indicate the smoothed tokens lead to better learning 
behavior and faster convergence, which supports its stabilization ability.
\begin{figure}[h]
\centering
\begin{subfigure}{0.24\textwidth}
    \includegraphics[width=\textwidth]{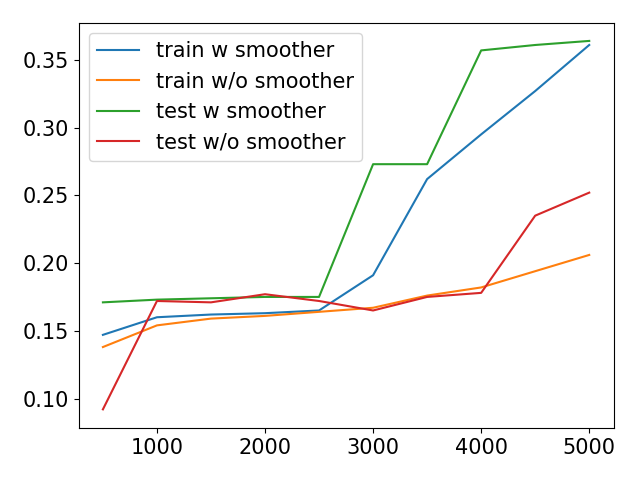}
    \caption{Listops Accuracy}
\end{subfigure}
\hfill
\begin{subfigure}{0.24\textwidth}
    \includegraphics[width=\textwidth]{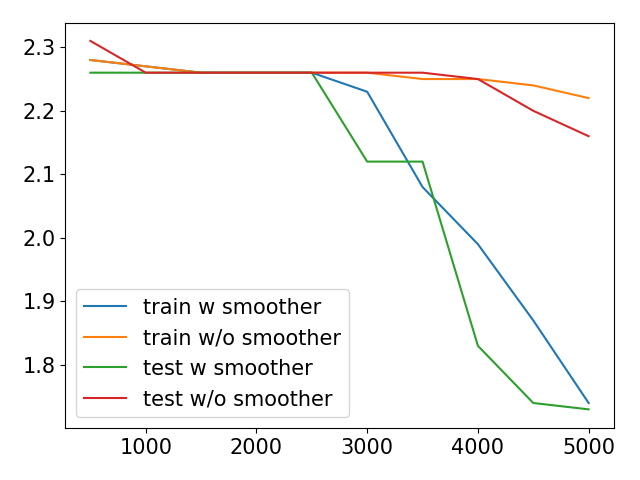}
    \caption{Listops Loss}
\end{subfigure}
        

\centering
\begin{subfigure}{0.24\textwidth}
    \includegraphics[width=\textwidth]{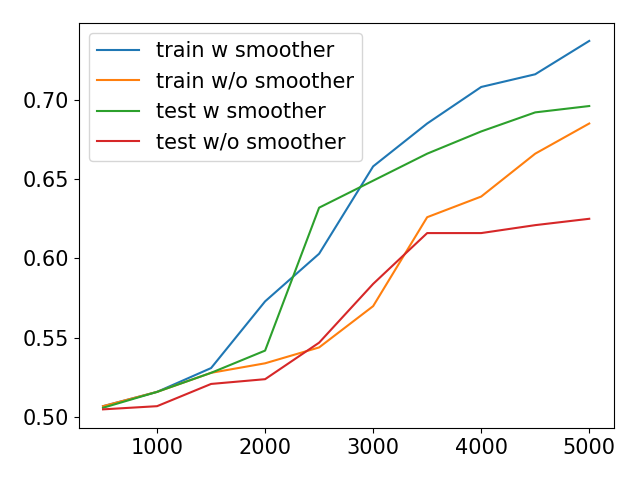}
    \caption{Text Accuracy}
\end{subfigure}
\hfill
\begin{subfigure}{0.24\textwidth}
    \includegraphics[width=\textwidth]{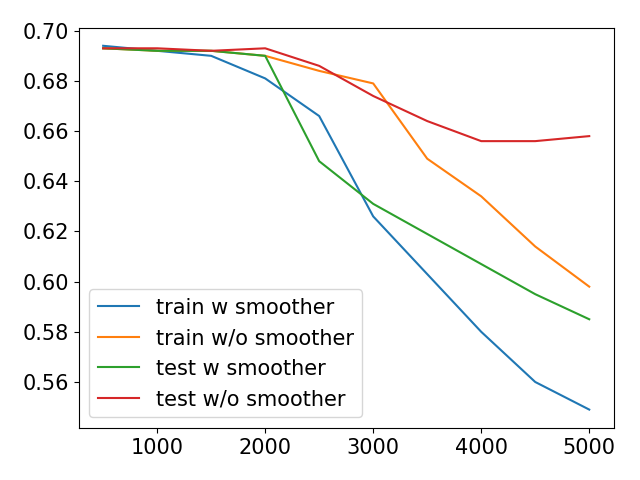}
    \caption{Text Loss}
\end{subfigure}
        

\centering
\begin{subfigure}{0.24\textwidth}
    \includegraphics[width=\textwidth]{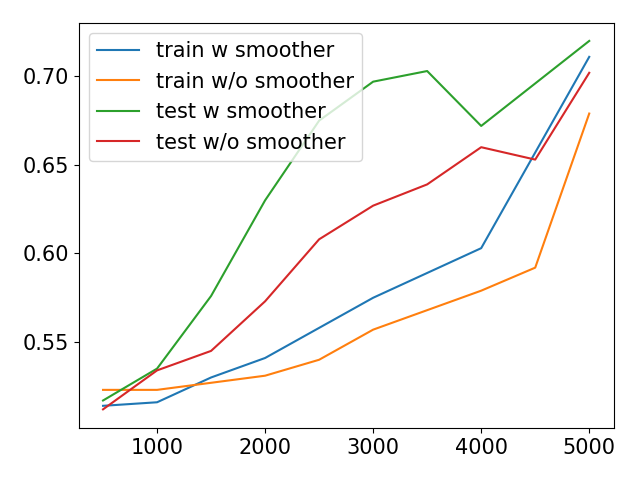}
    \caption{Retrieval Accuracy}
\end{subfigure}
\hfill
\begin{subfigure}{0.24\textwidth}
    \includegraphics[width=\textwidth]{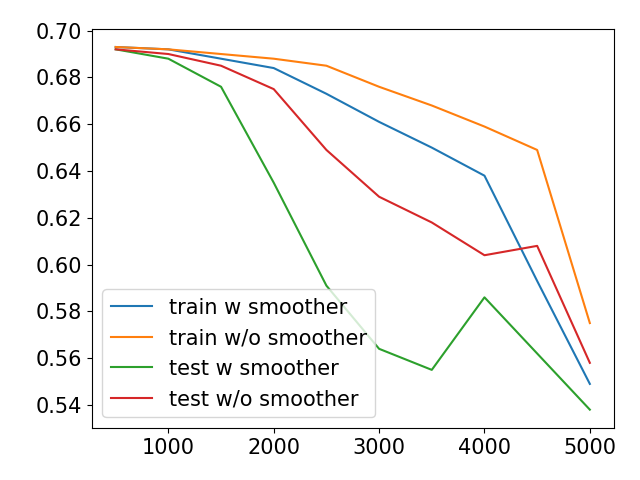}
    \caption{Retrieval Loss}
\end{subfigure}
        

\centering
\begin{subfigure}{0.24\textwidth}
    \includegraphics[width=\textwidth]{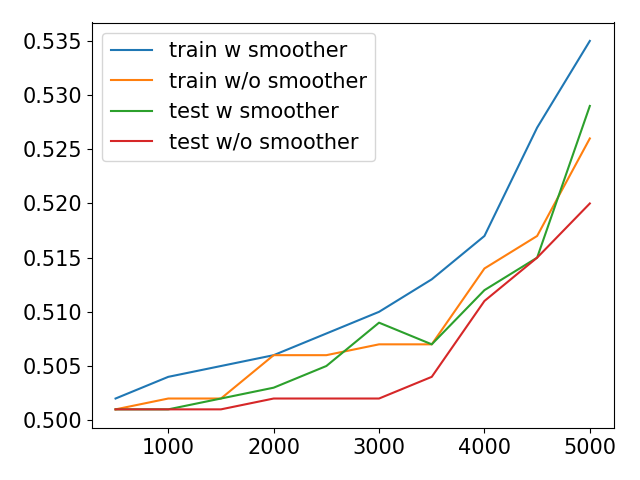}
    \caption{Pathfinder Accuracy}
\end{subfigure}
\hfill
\begin{subfigure}{0.24\textwidth}
    \includegraphics[width=\textwidth]{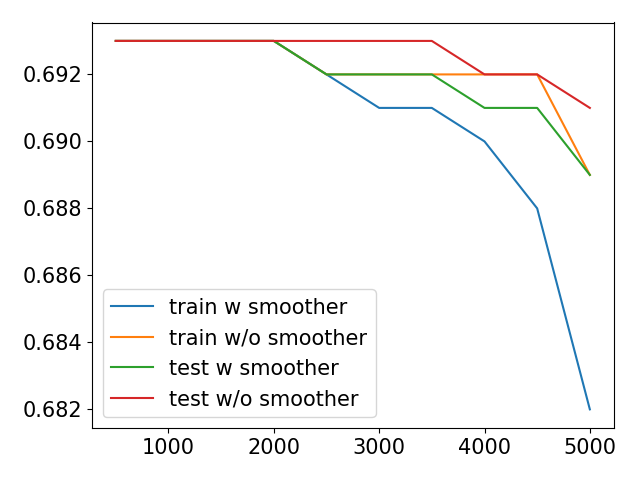}
    \caption{Pathfinder Loss}
\end{subfigure}
        

\centering
\begin{subfigure}{0.24\textwidth}
    \includegraphics[width=\textwidth]{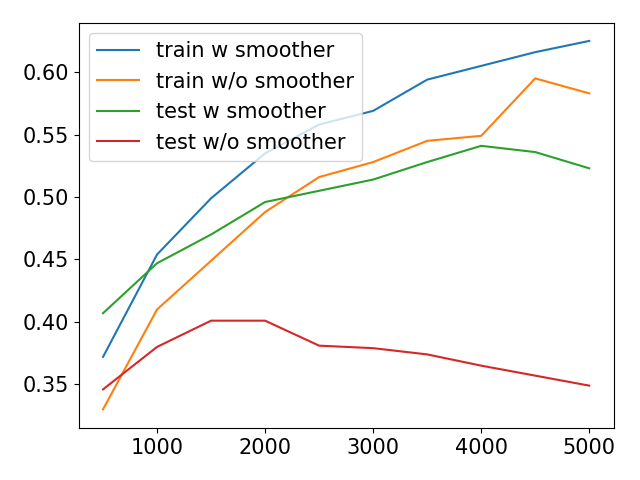}
    \caption{Image Accuracy}
\end{subfigure}
\hfill
\begin{subfigure}{0.24\textwidth}
    \includegraphics[width=\textwidth]{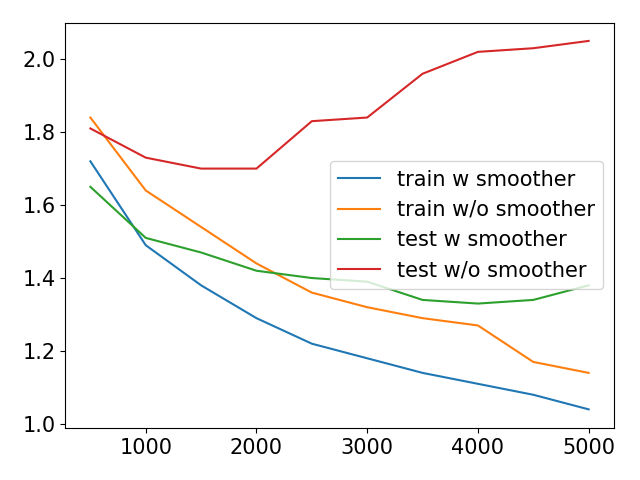}
    \caption{Image Loss}
\end{subfigure}
        
\caption{Learning Curve for LRA experiments.}
\label{fig:learning_curve}
\end{figure}

\subsection{Ablation Study}
This subsection provides an ablation test on four components: Fourier Convolution, Convolution Stem, Column Attention, and Row Attention. We use S$^3$Attention with ($r, s_1, s_2 = 8$) as the baseline, and the settings are detailed in Table~\ref{tab:ablation} in the supplementary material. In Table~\ref{tab:Ablations}, we present the accuracy changes when removing each component. The performance-decreasing results indicate all four components of S$^3$Attention are necessary to reach promising results. The most significant component is Column Attention which leads 8.28 average accuracy difference. It states that a good summary of the whole sequence is important. Similar observations are also reported in Transformer-LS \citep{zhu2021long} and XCiT \citep{ali2021xcit}, where the spirit of Attention over columns is used in the dynamic project and Cross-Covariance Attention respectively. The second most effective part is Fourier Convolution. It reaches a 13.89\% accuracy difference in the Retrieval task involving two 4k sequences. Fourier Convolution also works well on shorter sequence tasks (e.g., Image and Pathfinder) and brings a 6.12\% accuracy difference. 

\begin{table}
\centering
\caption{Ablation experiments. The S$^3$Attention $(r,s_1,s_2 = 8)$ is used as baseline.  The differences by removing each component from the baseline model are reported.}
\scalebox{0.92}{
\begin{tabular}{l|ccccc|c}
\toprule
Model& LisOps& Text & Retrieval & Image & Pathfinder& Average\\
\midrule
Baseline	    &38.30	&69.27	&83.26	&53.90	&75.82	&64.11\\
\midrule
Four.~Conv.	&-0.47	&-4.04	&-13.98	&-5.64	&-6.59	&-6.14\\
Conv.~Stem	    &-0.13	&-0.55	&-1.51	&-1.76	&{-0.47}	&-0.88\\
Col.~Attn.	    &-1.16	&-8.00	&-9.16	&-10.63	&-12.45	&-8.28\\
Row~Attn.	    &-0.38	&-1.92	&-1.97	&-2.64	&-2.56	&-1.89\\
\bottomrule
\end{tabular}
}
\label{tab:Ablations}
\end{table}

\section{Concluding Remarks} \label{sec:conclusion}
We propose S$^3$Attention, a robust and efficient Attention architecture for modeling long sequences with a good balance between feature preserving and noise resistance. It aggregates a Fourier convolutional stem smoothing information among tokens and a Skeleton-Sketching-inspired efficient Attention. In particular, our proposed Skeleton Attention directly samples the columns and rows of the token matrix. Such a design increases the model's robustness and gives us a positive near-linear complexity side effect. We conduct a thorough theoretical and experimental analysis of the proposed model and show its effectiveness. Lastly, extensive experiments show that the proposed model achieves the best performance on Long Range Arena and a state-of-art performance in long-term time series forecasting tasks compared to various Attention-based baselines.

One limitation of the current S$^3$Attention is that we need to use both FFT and IFFT in a sequential manner, which is potentially slower than the existing Fourier-based Attention architectures (e.g., \citep{lee2021fnet}) that only involve the FFT. As our primary goal using Fourier convolution is to smooth the token matrix and reduce the incoherent parameter, we can use Random Fourier Transformation \citep{ailon2006approximate} to modify S$^3$Attention with only FFT. Another limitation is that the size of $\bm{L}$ matrix in the Fourier Convolution part is the same as the input sequences. On longer sequences, $\bm{L}$ will contain more learnable parameters that make the model easier to overfit. We may introduce low rankness or use a more sophisticated design, such as \citep{gu2021efficiently}, to tackle this issue in the future.

Another future research direction is to incorporate the breakthroughs in state-space model fields (e.g., \citep{gu2020hippo,gu2021efficiently,gupta2022diagonal,ma2023mega}). Those works obtain very competitive performance and even beat various Attention variant architectures by a large margin in several long sequence tasks. In recent work \citep{wang2022pretraining}, the state space model even reaches comparable performance in large-scale NLP pretraining tasks to the BERT-type models. It is possible to further boost S$^3$Attention's performance by combining with those breakthroughs.


\bibliographystyle{IEEEtran}
\bibliography{IEEEabrv,reference}


\newpage
\input{appendix.tex}

\vfill

\end{document}

%% file: appendix.tex

\setcounter{page}{1}
\renewcommand{\thepage}{S-\arabic{page}}
\setcounter{table}{0}
\renewcommand{\thetable}{S\arabic{table}}

\onecolumn
\begin{center}
\huge
Supplementary Material for \\
S$^3$Attention: Improving Long Sequence Attention with Smoothed
Skeleton Sketching  
\end{center}

\appendices
\subsection{Algorithms}
\label{Algorithms_all}

In this section, the pseudo-codes of the proposed methods are presented in Algorithms 1 and 2. They give the model architectures of the Skeleton Attention and Smoother component. Algorithms~3 and 4 give the details for forecasting experiments in Section~\ref{sec:forecasting}. The implemented code is available online at \url{https://github.com/wxie9/S3Attention}. 
\input{Algos/algorithm1}
\input{Algos/algorithm2}
\input{Algos/algorithm3}

\subsection{Experiment Configurations}\label{sec:exp_config}
In this section, we report the configurations in Tables~\ref{tab:config_888}-\ref{tab:ablation} for the experiments in Section~\ref{sec:experiment}.
\begin{table}[H]
\centering
\caption{Experiment Configuration of S$^3$Attention $(r,s_1,s_2=8)$.}
\scalebox{1}{
\begin{tabular}{l|ccccc}
\toprule
Parameters& ListOps& Text & Retrieval & Image & Pathfinder\\
\midrule
Epoch                   &5          &30     &15      &60      &100  \\
Learning Rate           &1e-4       &1e-4     &1e-4      &1e-3      &1e-4   \\
Weight Decay            &0          &1e-2     &1e-2      &1e-2      &1e-2   \\
Batch Size              &32         &32     &32      &256      &256  \\
$r,s_1,s_2$                     &8,8,8           &8,8,8       &8,8,8       &8,8,8        &8,8,8    \\
dropout in embedding    &0          &0.5     &0.1      &0.1      &0   \\
dropout in attention    &0          &0.1     &0.1      &0.1      &0   \\
dropout in smoother     &0          &0.5     &0.1      &0.5      &0.5   \\
\bottomrule
\end{tabular}
}
\label{tab:config_888}
\end{table}
\begin{table}[H]
\centering
\caption{Experiment Configuration of S$^3$Attention (best).}
\scalebox{1}{
\begin{tabular}{l|ccccc}
\toprule
Parameters& ListOps& Text & Retrieval & Image & Pathfinder\\
\midrule
Epoch                   &10         &30     &15      &60      &100  \\
Learning Rate           &1e-4       &1e-4     &1e-4      &1e-3      &1e-4   \\
Weight Decay            &1e-2          &1e-2     &1e-2      &1e-2      &1e-2   \\
Batch Size              &32         &32     &32      &256      &256  \\
$r,s_1,s_2$                     &8,8,8             &8,8,8       &8,32,32      &8,16,16     &8,128,32  \\
dropout in embedding    &0          &0.5     &0.1      &0.5      &0.1   \\
dropout in attention    &0          &0.1     &0.1      &0.1      &0.1   \\
dropout in smoother     &0          &0.5     &0.1      &0.5      &0.5   \\
\bottomrule
\end{tabular}
}
\label{tab:config_best}
\end{table}
\begin{table}[H]
\centering
\caption{Experiment Configuration for Model Parameters Impact.}
\scalebox{1}{
\begin{tabular}{l|ccccc}
\toprule
Parameters& ListOps& Text & Retrieval & Image & Pathfinder\\
\midrule
Epoch                   &5          &30     &15      &60      &100  \\
Learning Rate           &1e-4       &1e-4     &1e-4      &1e-3      &1e-4   \\
Weight Decay            &0          &1e-2     &1e-2      &1e-2      &1e-2   \\
Batch Size              &32         &32     &32      &256      &256  \\
dropout in embedding    &0          &0.5     &0.1      &0.1      &0   \\
dropout in attention    &0          &0.1     &0.1      &0.1      &0   \\
dropout in smoother     &0          &0.5     &0.1      &0.5      &0.5   \\
\bottomrule
\end{tabular}
}
\label{tab:impact_params}
\end{table}
\begin{table}[h]
\centering
\caption{Experiment Configuration for Ablation.}
\scalebox{1}{
\begin{tabular}{l|ccccc}
\toprule
Parameters& ListOps& Text & Retrieval & Image & Pathfinder\\
\midrule
Learning Rate           &1e-4       &1e-4     &1e-4      &1e-3      &1e-4   \\
Weight Decay            &0          &1e-2     &1e-2      &1e-2      &1e-2   \\
Batch Size              &32         &32     &32      &256      &256  \\
$r,s_1,s_2$                      &8,8,8            &8,8,8        &8,8,8         &8,8,8         &8,8,8    \\
dropout in embedding    &0          &0.5     &0.1      &0.1      &0   \\
dropout in attention    &0          &0.1     &0.1      &0.1      &0   \\
dropout in smoother     &0          &0.5     &0.1      &0.5      &0.5   \\
\bottomrule
\end{tabular}
}
\label{tab:ablation}
\end{table}


\subsection{Additional Results on LRA}
We have already provided the average of five runs with different random seeds in Table~\ref{tab:lra-benchmarks}. Here we also provide the standard deviations for these experiments in Table~\ref{tab:error bar}.

\begin{table}[h]
\centering
\caption{ Accuracy on Long Range Arena (LRA) with standard errors shown in parenthesis. All results are averages of 5 runs with different random seeds.}
\scalebox{1}{
\begin{tabular}{l|ccccc}
\toprule
Model& LisOps& Text & Retrieval & Image & Pathfinder\\
\midrule
S$^3$Attention ($r,s_1,s_2= 8$) & 38.30 (0.40)&69.27 (0.83)&83.26 (0.45) &53.90 (1.54) &75.82 (0.97)\\ 
S$^3$Attention (best)  & 39.15 (0.48)& 71.58 (0.95)&83.73 (0.61)  &57.73 (1.83) &78.20 (1.32)\\ 
\bottomrule
\end{tabular}
}
\label{tab:error bar}
\end{table}

\subsection{Dataset and Implementation Details}
In this section, we summarize the details of the datasets used in this paper as follows: 

LRA datasets: {\bf ListOps}(2K length mathematical expression task which investigates the parsing ability); {\bf Text}
(up to 4K byte/character-level document classification task that tests capacity in character compositionality); {\bf Retrieval} (byte/character-level document matching task, which examines the information compression ability with two 4K length sequences); {\bf Image} (pixel-wise sequence image classification based on the CIFAR-10 dataset); {\bf Pathfinder} (long-range spatial dependency identification task. The input images contain two small
points/circles and dash-line paths. The model needs to identify whether two points/circles are connected); The LRA has several desirable advantages that made us focus on it as the evaluation benchmark: {\bf generality} (only requires the encoder part); {\bf simplicity} (data augmentation and
pretraining are out of scope); {\bf challenging long inputs} (difficulty enough and room to improve); {\bf  diversity aspects} (tasks covering math, language, image, and spatial modeling);
and {\bf lightweight} (run with low resource requirement). 

Time series datasets:1) ETT~\citep{haoyietal-informer-2021} dataset contains two sub-datasets: ETT1 and ETT2, collected from two separated counties. Each of them has two versions of sampling resolutions (15min \& 1h). ETT dataset contains multiple time series of electrical loads and one time sequence of oil temperature. 2) Electricity\footnote{\url{https://archive.ics.uci.edu/ml/datasets/ElectricityLoadDiagrams20112014}.} dataset contains the electricity consumption for more than three hundred clients with each column corresponding to one client. 
\renewcommand{\citenumfont}[1]{S#1}
3) Exchange~\citep{lai2018modeling-exchange-dataset} dataset contains the current exchange of eight countries. 
\renewcommand{\citenumfont}[1]{#1}
4) Traffic\footnote{\url{http://pems.dot.ca.gov}.} dataset contains the occupation rate of freeway systems in California, USA. 5) Weather\footnote{\url{https://www.bgc-jena.mpg.de/wetter}.} dataset contains 21 meteorological indicators for a range of one year in Germany. 6) Illness\footnote{\url{https://gis.cdc.gov/grasp/fluview/fluportaldashboard.html}.} dataset contains the influenza-like illness patients in the United States. Table~\ref{tab:dataset} 
summarizes all the features of the six benchmark datasets. They are all split into the training set, validation set and test set by the ratio of 7:1:2 during modeling. 

\input{tables/tab_dataset}

GLUE datasets:  The GLUE benchmark covers various natural language understanding tasks and is widely used in evaluating transferring ability. The tasks can be divided into two types, single-sentence tasks (SST-2 and CoLA), and sentence-pair tasks (MNLI, QQP, QNLI, STS-B, MRPC, and RTE). Following the same settings in \citep{devlin2018bert}, we exclude WNLI task.

\subsection{Experiments on the Smoothness Effect of Fourier Convolution}
In this section, we verify Fourier convolution component in the Smoother block can reduce the incoherence value in the early training stage. We use S$^3$Attention with ($r, s_1, s_2 = 8$) as the test model and test on an NLP dataset: Text, and a vision dataset: Pathfinder. We compute the  $\mu$-incoherence value. of the token matrix before and after the Fourier convolution (denoted as $\mu_{\bm{X}}$ and $\mu_{\bm{X}^{\mathrm{smooth}}}$, respectively) for each sample in the validation dataset. Since we do not explicitly force the token matrix to be low-rank required by Definition~\ref{def:incoh}, we report the incoherence value for different rankness settings ($\mathrm{rank}  = 16$ and $\mathrm{rank}  = 32$) approximately, and the mean and standard deviation of incoherence value can be found in Table~\ref{tab:Incoherence}. The average incoherence value was reduced 30\% after the Fourier convolution in both datasets. Moreover, We observe that the standard deviation significantly decreases, which suggests the Fourier convolution may also potentially stabilize the training procedure.

\begin{table}[H]
\centering
\caption{The average incoherence parameters after 100 training steps with standard errors shown in the parenthesis.}
\scalebox{1}{
\begin{tabular}{l|cccc}
\toprule
Dataset& $\mu_{\bm{X}}\ (\mathrm{rank} = 32)$& $\mu_{\bm{X}^{\mathrm{smooth}}}\ (\mathrm{rank}  = 32)$ & $\mu_{\bm{X}}\ (\mathrm{rank}  = 16)$& $\mu_{\bm{X}^{\mathrm{smooth}}}\ (\mathrm{rank}  = 16)$\\
\midrule
Text& 2.75 (0.027) & 2.05 (0.007)&3.98 (0.046)& 3.23 (0.038)\\
Pathfinder& 3.83 (0.221)& 1.99 (0.001)&4.88 (0.264)&3.48 (0.001) \\
\bottomrule
\end{tabular}
}
\label{tab:Incoherence}
\end{table}

\begin{table}[H]
\centering
\caption{The training configurations for Pretraining and GLUE tasks}
\scalebox{1}{
\begin{tabular}{l|cc}
\toprule
& Pre-training&GLUE \\
\midrule
Max Steps&1000K&-\\
Max Epochs&-&[4,20]\\
Learning Rate&1e-4&[5e-5,1e-4]\\
Batch Size&256&[16,32]\\
Warm-up Steps&5000&-\\
Sequence Length&512&128\\
Learning Rate Decay&-&Linear\\
Clip&-&1\\
Dropout&-&0.1\\
\bottomrule
\end{tabular}
}
\label{tab:BERT configure}
\end{table}






\renewcommand{\bibnumfmt}[1]{[S#1]}

\renewcommand{\refname}{Additional References}

%% file: Algos/algorithm1.tex
\begin{algorithm}[h]
	\caption{Skeleton Attention}
	\label{alg:code1}
	\definecolor{codeblue}{rgb}{0.25,0.5,0.5}
	\lstset{
		backgroundcolor=\color{white},
		basicstyle=\fontsize{7.2pt}{7.2pt}\ttfamily\selectfont,
		columns=fullflexible,
		breaklines=true,
		captionpos=b,
		commentstyle=\fontsize{7.2pt}{7.2pt}\color{codeblue},
		keywordstyle=\fontsize{7.2pt}{7.2pt},
	}
    
    \begin{lstlisting}[language=python]
   class Skeleton_Attention(nn.Module):
    def __init__(self, num_head = 2, head_dim = 32,seq_len, left_rank = 8,right_rank = 8, dropout = 0.1):
        super(Skeleton_Attention, self).__init__()
        self.num_head = num_head
        self.head_dim = head_dim
        self.seq_len = seq_len
        self.left_rank = left_rank
        self.right_rank = right_rank
        
        self.ln_1 = nn.LayerNorm(self.num_head * self.head_dim)
        self.ln_2 = nn.LayerNorm(self.num_head * self.head_dim)

        self.drop_attn = torch.nn.Dropout(p=dropout)        
        
        self.index_set_right =   torch.randperm(self.head_dim)
        self.index_set_right = self.index_set_right[:self.right_rank] 
        
        self.index_set_left =   torch.randperm(self.seq_len)
        self.index_set_left = self.index_set_left[:self.left_rank]

    def combine_heads(self, X):
        X = X.transpose(1, 2)
        X = X.reshape(X.size(0), X.size(1), self.num_head * self.head_dim)
        return X

    def split_heads(self, X):
        X = X.reshape(X.size(0), X.size(1), self.num_head, self.head_dim)
        X = X.transpose(1, 2)
        return X
        
    def forward(self,Q, K, V):
        #### Row Attention ####        
        if self.left_rank <= self.seq_len:
            K1 = K[:,:,self.index_set_left,:]
            V1 = V[:,:,self.index_set_left,:]
        else:
            K1 = K
            V1 = V

        dots = Q @ K1.transpose(-1,-2)  
        dots = dots / math.sqrt(self.head_dim)
        attn = nn.functional.softmax(dots,dim=-1)
        attn = self.drop_attn(attn)
        
        #### Column Attention ####          
        Q2 = Q.transpose(-1,-2)
        if self.right_rank <= self.head_dim:

            K2 = K[:,:,:,self.index_set_right]
            V2 = V[:,:,:,self.index_set_right]
        else:
            K2 = K
            V2 = V
    
        dots_r = Q2 @ K2
        dots_r = dots_r / math.sqrt(self.seq_len)
        attn_r = nn.functional.softmax(dots_r,dim=-1).transpose(-1,-2)
        attn_r = self.drop_attn(attn_r)

        X = self.split_heads(self.ln_1(self.combine_heads(torch.matmul(attn,V1))))/2 + self.split_heads(self.ln_2(self.combine_heads(torch.matmul(V2,attn_r))))/2
 
        return X
    \end{lstlisting}
\end{algorithm}

%% file: Algos/algorithm2.tex
\begin{algorithm}[h]
	\caption{Smoother component}
	\label{alg:code2}
	\definecolor{codeblue}{rgb}{0.25,0.5,0.5}
	\lstset{
		backgroundcolor=\color{white},
		basicstyle=\fontsize{7.2pt}{7.2pt}\ttfamily\selectfont,
		columns=fullflexible,
		breaklines=true,
		captionpos=b,
		commentstyle=\fontsize{7.2pt}{7.2pt}\color{codeblue},
		keywordstyle=\fontsize{7.2pt}{7.2pt},
	}
    
    \begin{lstlisting}[language=python]
class Smoother(nn.Module):

    def __init__(self, hidden_size, seq_len, dropout = 0.5, num_head = 2,transformer_dim = 64, fold = 1):

        super(Smoother, self).__init__()

        self.hidden_size = hidden_size
        self.seq_len = seq_len
        self.dropout = dropout
        self.num_head = num_head
        self.dim = transformer_dim
        self.fold = fold
        
        self.weights_fft = nn.Parameter(torch.empty(self.seq_len//2+1, self.hidden_size,2))
        nn.init.kaiming_normal_(self.weights_fft, mode='fan_in', nonlinearity='relu')
        
        self.tiny_conv_linear =  torch.nn.Conv1d(in_channels = self.hidden_size*2 , out_channels = self.hidden_size, kernel_size = 3, padding=  1, groups = 1)
        self.dropout = torch.nn.Dropout(p=self.dropout)
        self.bn_1 = nn.BatchNorm1d(self.seq_len)
        
    def forward(self, x):
  
        #### Compute Segment Average ####    
        B,S,H = x.shape
        u = x.reshape(B,S,self.fold,H//self.fold)
        u = torch.mean(u,dim = -1)
        
        #### Fourier Convolution ####   
        fft_u = fft.rfft(u, n =  self.seq_len, axis = -2)
        fft_u = torch.view_as_real(fft_u)
        fft_u = fft_u.repeat(1,1,H//self.fold,1)
        self.weight_used = self.weights_fft.unsqueeze(0)
        temp_real = fft_u[...,0]*self.weight_used[...,0] - fft_u[...,1]*self.weight_used[...,1]
        temp_imag = fft_u[...,0]*self.weight_used[...,1] + fft_u[...,1]*self.weight_used[...,0]
        out_ft = torch.cat([temp_real.unsqueeze(-1),temp_imag.unsqueeze(-1)],dim =  -1)
        out_ft = torch.view_as_complex(out_ft) 
        m = fft.irfft(out_ft, n =  self.seq_len, axis = -2)
                      
        #### Convolution Stem #### 
        input_h = torch.cat((m, x), dim = -1) 
        h =  self.tiny_conv_linear(input_h.permute(0,2,1)).permute(0,2,1)
        h = self.dropout(F.relu(self.bn_1(h)))
    
        return   h
    \end{lstlisting}
\end{algorithm}

%% file: Algos/algorithm3.tex
\begin{algorithm}[h]
	\caption{pseudo code for  Time-Series Forecasting}
	\label{alg:code3}
	\definecolor{codeblue}{rgb}{0.25,0.5,0.5}
	\lstset{
		backgroundcolor=\color{white},
		basicstyle=\fontsize{7.2pt}{7.2pt}\ttfamily\selectfont,
		columns=fullflexible,
		breaklines=true,
		captionpos=b,
		commentstyle=\fontsize{7.2pt}{7.2pt}\color{codeblue},
		keywordstyle=\fontsize{7.2pt}{7.2pt},
	}
    
    \begin{lstlisting}[language=python]
    def forward(self, x_in):
        B1,H1,C1 = x_in.shape
        for i in range(len(self.encoder)):
            attn_layer = self.encoder[i]
            #standardize the input data
            if i == 0: 
                tmp_mean = torch.mean(x_in[:,:,:],dim = 1,keepdim = True)
                tmp_std = torch.sqrt(torch.var(x_in[:,:,:],dim = 1,keepdim = True)+1e0)
                x_in = (x_in - tmp_mean)/(tmp_std) 

                enc_out1 = self.enc_embedding(x_in)
         
            enc_out1= attn_layer(enc_out1) + enc_out1 
        
        #decoder via Fourier Extrapolation
        dec_out = self.fourierExtrapolation(post(enc_out1))
        output = (dec_out.reshape(B1,-1,C1))*(tmp_std)+tmp_mean 
        return  output
    \end{lstlisting}
\end{algorithm}

\begin{algorithm}[h]
	\caption{Fourier Extrapolation}
	\label{alg:code4}
	\definecolor{codeblue}{rgb}{0.25,0.5,0.5}
	\lstset{
		backgroundcolor=\color{white},
		basicstyle=\fontsize{7.2pt}{7.2pt}\ttfamily\selectfont,
		columns=fullflexible,
		breaklines=true,
		captionpos=b,
		commentstyle=\fontsize{7.2pt}{7.2pt}\color{codeblue},
		keywordstyle=\fontsize{7.2pt}{7.2pt},
	}
    
    \begin{lstlisting}[language=python]
class fourierExtrapolation(nn.Module):
    def __init__(self,inputSize,n_harm = 8,n_predict = 96):
        super().__init__()
        self.n = inputSize
        self.n_harm = n_harm
        self.f = torch.fft.fftfreq(self.n)     
        self.indexes = list(range(self.n))

        # sort indexes by frequency, lower -> higher
        self.indexes.sort(key = lambda i: torch.absolute(self.f[i]))
        self.indexes = self.indexes[:1 + self.n_harm * 2]

        self.n_predict = n_predict

        # compute init phase 
        self.t = torch.arange(0, self.n + self.n_predict)
        self.t1 = self.t.unsqueeze(0).unsqueeze(-1).float().to('cuda')
        self.f = self.f.unsqueeze(0).unsqueeze(-1).to('cuda')
        self.t = self.t.unsqueeze(0).unsqueeze(-1).unsqueeze(-1)to('cuda')
        self.g = self.f[:,self.indexes,:].permute(0,2,1).unsqueeze(1)
        self.phase_init = 2 * 3.1415 * self.g * self.t

    def fourierExtrapolation(self,x):

        # x in frequency domain
        x_freqdom = torch.fft.fft(x,dim = -2)    
        x_freqdom = torch.view_as_real(x_freqdom)
        # select importance frequencies 
        x_freqdom = x_freqdom[:,self.indexes ,:,:]
        x_freqdom = torch.view_as_complex(x_freqdom)
        ampli = torch.absolute(x_freqdom) / self.n   # amplitude
        phase = torch.angle(x_freqdom)          # phase

        ampli = ampli.permute(0,2,1).unsqueeze(1)
        phase = phase.permute(0,2,1).unsqueeze(1)

        self.restored_sig = ampli * torch.cos(self.phase_init + phase)

        return torch.sum(self.restored_sig,dim = -1)
    \end{lstlisting}
\end{algorithm}

%% file: tables/tab_dataset.tex
\begin{table}[H]
\caption{Details of time series benchmark datasets.}
\label{tab:dataset}
\begin{center}
\begin{small}
\begin{sc}
\begin{tabular}{l|cccr}
\toprule
Dataset & Length & Dimension & Frequency \\
\midrule
ETTm2 & 69680 & 8 & 15 min\\
Exchange & 7588 & 9 & 1 day\\
Weather & 52696 & 22 & 10 min & \\
Electricity & 26304 & 322 & 1h & \\
ILI & 966 & 8 & 7 days\\
Traffic & 17544 & 863 & 1h & \\
\bottomrule
\end{tabular}
\end{sc}
\end{small}
\end{center}
\vskip -0.1in
\end{table}